\documentclass[accepted,mathfont=newtx]{uai2025} 

\usepackage[american]{babel}

\usepackage{natbib} 
\usepackage{mathtools} 
\usepackage{booktabs} 
\usepackage{tikz} 
\usepackage{pifont}
\usepackage{microtype}
\usepackage{graphicx}
\usepackage{accents}
\usepackage{subfigure}
\usepackage{hyperref}

\usepackage{amsmath}

\usepackage{amsthm}
\usepackage{subcaption}
\usepackage[capitalize,noabbrev]{cleveref}
\usepackage[textsize=tiny]{todonotes}
\usepackage{algpseudocode}
\usepackage{algorithm}
\usepackage[]{authblk}

\usepackage{amsmath,amsfonts,bm}

\def\eqref#1{equation~\ref{#1}}

\def\1{\bm{1}}

\DeclareMathAlphabet{\mathsfit}{\encodingdefault}{\sfdefault}{m}{sl}
\SetMathAlphabet{\mathsfit}{bold}{\encodingdefault}{\sfdefault}{bx}{n}

\def\gA{{\mathcal{A}}}
\def\gB{{\mathcal{B}}}

\def\gD{{\mathcal{D}}}

\def\gH{{\mathcal{H}}}

\def\gK{{\mathcal{K}}}
\def\gL{{\mathcal{L}}}

\def\gS{{\mathcal{S}}}
\def\gT{{\mathcal{T}}}

\def\gX{{\mathcal{X}}}
\def\gY{{\mathcal{Y}}}

\newcommand{\E}{\mathbb{E}}

\newcommand{\R}{\mathbb{R}}

\DeclareMathOperator*{\argmax}{arg\,max}

\DeclareMathOperator*{\argmin}{arg\,min}

\usepackage{thm-restate}
\usepackage{url}
\usepackage{adjustbox}

\usepackage{listings}
\usepackage{colortbl}
\usepackage{float}
\usepackage{xspace}
\usepackage{multirow}
\usepackage{enumitem}
\setlist{leftmargin=4mm,itemsep=0pt,topsep=0pt}
\usepackage{selectp}

\makeatletter
\def\blfootnote{\gdef\@thefnmark{}\@footnotetext}
\makeatother

\colorlet{alternateRowColor}{magenta!10}

\iftrue
\newcommand{\enote}[1]{\textcolor{red}{[E: #1]}}
\newcommand{\anote}[1]{\textcolor{blue}{[A: #1]}}
\newcommand{\pnote}[1]{\textcolor{orange}{[P: #1]}}
\else
\newcommand{\enote}[1]{}
\newcommand{\anote}[1]{}
\newcommand{\pnote}[1]{}
\fi

\newcommand{\disloss}{\ell_{\textrm{dis}}}
\newcommand{\logloss}{\ell_{\textrm{logistic}}}

\newcommand{\disdis}{\textsc{Dis}$^2$\xspace}
\newcommand{\odd}{\textsc{ODD}\xspace}

\theoremstyle{plain}
\newtheorem{theorem}{Theorem}[section]

\newtheorem{lemma}[theorem]{Lemma}

\theoremstyle{definition}
\newtheorem{definition}[theorem]{Definition}
\newtheorem{assumption}[theorem]{Assumption}
\theoremstyle{remark}

\begin{document}

\title{ODD: Overlap-aware Estimation of Model Performance under Distribution Shift}

%
%
\author[1]{\href{mailto:<amishr24@jhu.edu>?Subject=ODD paper}{Aayush Mishra}}
\author[1]{Anqi Liu}
\affil[1]{%
    Department of Computer Science, \protect\\
    Johns Hopkins University,
    Baltimore, Maryland, USA
     \protect\\
    \texttt{\{amishr24, aliu.cs\}@jhu.edu}
}
  
\maketitle

\begin{abstract}
Reliable and accurate estimation of the error of an ML model in unseen test domains is an important problem for safe intelligent systems. Prior work uses \textit{disagreement discrepancy} (\disdis) to derive practical error bounds under distribution shifts. It optimizes for a maximally disagreeing classifier on the target domain to bound the error of a given source classifier. Although this approach offers a reliable and competitively accurate estimate of the target error, we identify a problem in this approach which causes the disagreement discrepancy objective to compete in the overlapping region between source and target domains.
With an intuitive assumption that the target disagreement should be no more than the source disagreement in the overlapping region due to high enough support, we devise Overlap-aware Disagreement Discrepancy (\odd). 
Our \odd-based bound uses domain-classifiers to estimate domain-overlap and better predicts target performance than \disdis. We conduct experiments on a wide array of benchmarks to show that our method improves the overall performance-estimation error while remaining valid and reliable. Our code and results are available on \href{https://github.com/aamixsh/odd}{GitHub}.
\end{abstract}

\begin{figure}[t]
    \centering
    \includegraphics[width=\linewidth]{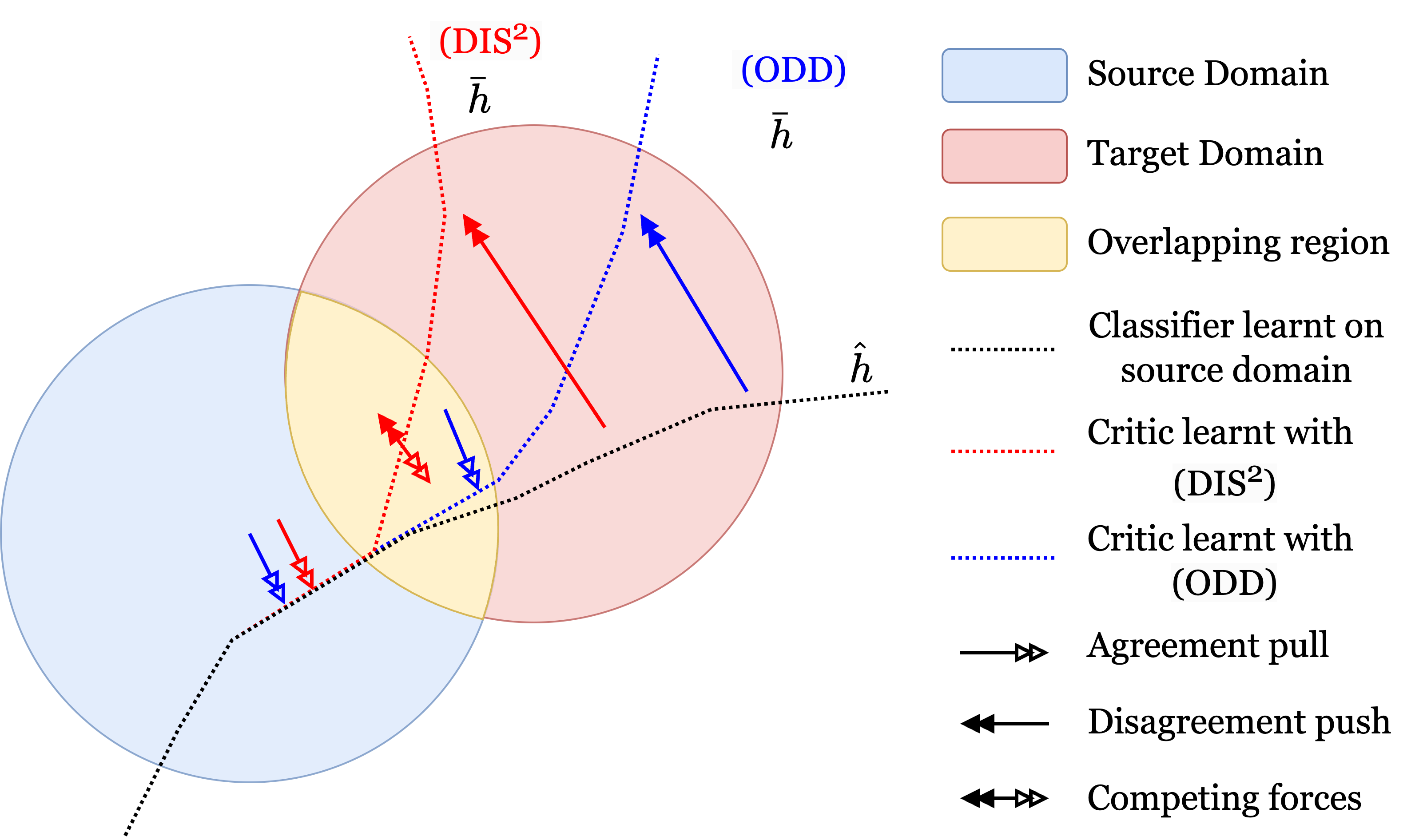}
    \caption{Given a source domain $\gS$, and a classifier $\hat{h}$ trained in this domain, we aim to predict the performance of $\hat{h}$ in a target domain $\gT$ \textit{without labels}. ~\citet{rosenfeld2023almost} learn a worst-case \textit{critic} $\Bar{h}$, which maximally disagrees with $\hat{h}$ in $\gT$ while agreeing with $\hat{h}$ in $\gS$ (\disdis), to bound this performance. But this optimization leads to a competition in the overlapping region. \odd discounts disagreement in the overlapping region making $\Bar{h}$ agree with $\hat{h}$ in this region. We show that \odd tightens the gap between true and predicted performance compared to \disdis, while remaining reliable.}
    \label{fig:fig1}
\end{figure}

\section{Introduction}
\label{sec:intro}
The ability of machine learning models to know when they do not know, is an important characteristic to make them safe for deployment, especially in high-stakes applications like the medical domain. Modern neural networks have an incredible capacity to learn complex functions but they are prone to catastrophic errors on inputs outside of their training distributions. Hence, it is crucial to be able to predict the performance of these models on distributionally shifted test domains.

Many works have attempted to address this problem of test performance prediction~\citep{bendavid2007analysis, bendavid2010theory, mansour2009domain, ben2010impossibility}. These methods provide uniform convergence bounds which are often vacuous in practice because we are interested in bounding the error of a single hypothesis. Other recent works attempt to use unlabeled test samples to make point-wise (per hypothesis) prediction of performance bounds in neural networks~\citep{lu2023predicting, baek2022agreement, garg2022leveraging, guillory2021predicting}. Although these methods achieve low performance prediction error on average, they lack reliability and often overestimate the performance, especially under large distribution shift (when reliability is most important). To address these challenges,~\citet{rosenfeld2023almost} proposed Disagreement Discrepancy (\disdis), which provides almost provable performance bounds using unlabeled test samples. It 
trains a \textit{critic} that minimizes its disagreement with a given classifier on the source distribution while maximizing its disagreement on the target domain under a sufficiently expressive hypothesis class. This allows \disdis to assume a worst-case shift in the target domain while remaining faithful to the trained classifier in the source domain. This approach gives non-vacuous bounds that almost always remain valid. 

In this work, we find that \disdis can be further improved using domain-overlap awareness. The agreement with source and disagreement with target creates a tension in the region of domain overlap, resulting in unstable optimization, which leads to more pessimistic critics than necessary. 
We are motivated by the intuition that if a classifier was trained on labeled samples from the target domain, its support in the overlapping region would be similar to what the source trained classifier had. Hence, we propose to discount the disagreement of the critic with target overlapping samples through a new training objective. 
Leveraging domain classifiers for estimate domain-overlap, our Overlap-aware Disagreement Discrepancy (\odd) removes the instability of optimization in the overlapping region and results in practically tighter performance bounds than \disdis. We visually illustrate our method in~\autoref{fig:fig1}. In summary:
\begin{itemize} 
    \item We derive a general version of the \disdis bound using the notion of ideal joint hypothesis, and show that theoretical performance bounds obtained using \odd are as tight as those obtained using \disdis.
    \item We design an \odd-based disagreement loss to find more \textit{optimistic critics} for practically estimating the performance bounds. 
    \item On several real datasets and training method combinations, \odd-based bounds are tighter than \disdis-based bounds, and still maintain reliable coverage. 
\end{itemize}

\section{The general \disdis bound}
\label{sec:background}

\paragraph{Setup:} We follow the notation setup used by ~\citet{rosenfeld2023almost}. Let $\gS, \gT$ denote the source and target distributions, respectively, over labeled inputs $(x,y) \in \gX \times \gY$, and let $\hat{\gS}$, $\hat\gT$ denote the empirically observed sets (with cardinality $n_S$ and $n_T$) of samples from these distributions. \textit{In the target domain $\gT$, we observe only the covariates and not the labels.} We let $p_{\gS}$ and $p_{\gT}$ denote the marginal distributions of covariates in the source and target domains. We consider classifiers $h:\gX \to \R^{|\gY|}$ which output a vector of logits ($h(x)_y$ denotes the logits of class $y$ in this vector), and let $\hat{h}$ denote the particular classifier (trained on $\hat{\gS}$) under study for which we want to bound the error in the target domain. We use $\gH$ to denote a hypothesis class of such classifiers. 
For a domain $\gD$ on $\gX$, let $\epsilon_{\gD}(h, h') := \E_{x \sim p_\gD}[\mathbf{1}\{\argmax_y h(x)_y \neq \argmax_y h'(x)_y\}]$ denote the one-hot disagreement between any classifiers $h$ and $h'$ on $\gD$. 

\textbf{The labeling function:} Let $y^*_\gS$ and $y^*_\gT$ denote the true labeling functions in the source and target domains, respectively. \citet{rosenfeld2023almost} assume a fixed $y^*$ to represent the true labeling function for all domains. In contrast, we use the notion of ideal joint hypothesis \citep{bendavid2010theory} to analyze \textit{adaptability} in the general case.

\begin{definition}
\label{def:ideal}
The ideal joint hypothesis is defined as:
\begin{align}
\label{eq:ideal}
    y^* := \argmin_{h \in \gH} \epsilon_{\gS}(h, y^*_\gS) + \epsilon_{\gT}(h, y^*_\gT)
\end{align}
\end{definition}
with the corresponding joint risk defined as $\lambda := \epsilon_{\gS}(y^*) + \epsilon_{\gT}(y^*)$. For brevity, we overload $\epsilon_{\gD}(h)$ to mean $\epsilon_{\gD}(h, y^*_\gD)$, i.e. the 0-1 error of classifier $h$ on distribution $\gD$. If $\lambda$ is high, no $h$ trained on the source domain can be expected to perform well on the target domain (low adaptability). 



Now, \citet{rosenfeld2023almost} define \textit{disagreement discrepancy} (\disdis) as follows.

\begin{definition}
    \disdis $\Delta(h, h')$ is the disagreement between any $h$ and $h'$ on $\gT$ minus their disagreement on $\gS$:
    \begin{align}
    \label{eq:dis2}
        \Delta(h, h') := \epsilon_{\gT}(h, h') - \epsilon_{\gS}(h, h').
    \end{align}
\end{definition}

This immediately implies the following lemma:

\begin{lemma}
\label{lemma:error-expansion}
    For any classifier $h$,\\
    $\epsilon_{\gT}(h) \leq \epsilon_{\gS}(h) + \Delta(h, y^*) + \lambda$.
\end{lemma}

The proof can be found in~\autoref{app:proofs}. This gives us a method to bound the target risk in terms of source risk and the discrepancy term $\Delta$. However, as $y^*$ is unknown, we can optimize for an alternate critic $\Bar{h}$ which would act as a proxy for the worst-case $y^*$. For this, ~\citet{rosenfeld2023almost} assume the following:
\begin{assumption}
    \label{ass:dis2concave}
    For a critic $\Bar{h}\in\gH$ which maximizes a concave surrogate to the empirical \disdis, $\Delta(\hat{h}, y^*) \leq \Delta(\hat{h}, \Bar{h})$.
\end{assumption}
This assumption is based on the practical observation that $y^*$ is not chosen adversarially with respect to $\hat{h}$. So it is reasonable that there exists another function $h^* \in \gH$ having higher disagreement discrepancy than $y^*$.
$\Bar{h}$ is a concave surrogate to $h^*$ that we can find empirically (which approaches $h^*$ with increasing sample size). Finally, we have the error bound:
\begin{theorem}[\disdis Bound]
    \label{thm:dis2bound}
    Under~\autoref{ass:dis2concave}, 
    with probability $\geq 1-\delta$,
    \begin{align*}
        \epsilon_{\gT}(\hat{h}) &\leq \epsilon_{\hat{\gS}}(\hat{h}) + \hat\Delta(\hat{h}, \Bar{h}) + \sqrt{\frac{(n_S + 4n_T) \log \frac{1}{\delta}}{2 n_S n_T}} + \lambda.
    \end{align*}
\end{theorem}

Here, $\epsilon_{\hat{\gS}}$ and $\hat{\Delta}$ denote the empirical versions of their corresponding population terms. The proof can be found in~\autoref{app:proofs}. Even with a suitable $\Bar{h}$, calculating $\lambda$ still requires access to target labels, which we do not have. Hence, we follow~\citet{rosenfeld2023almost} and fix a common $y^*$ across domains (as it happens in most practical cases). When $y^* = y^*_\gS = y^*_\gT$, $\lambda = 0$ and the above bound collapses exactly to their bound. In the next section, we make \disdis more effective with overlap awareness.


\section{\odd: Overlap-aware \disdis}
\label{sec:odd}

\textbf{The problem with disagreement discrepancy:} Maximizing \disdis with respect to a given classifier $\hat{h}$ loosely translates to finding a \textit{critic} $\Bar{h}$ which maximally disagrees with $\hat{h}$ in the target domain (first term) while maximally agreeing with $\hat{h}$ in the source domain (second term). 
This works as expected in the case when the source and target domains are disjoint. However, there is often an overlap between source and target domains in practice. Maximizing \disdis creates a competition between the two terms in the region of \textit{domain overlap}. In this region, the maximizing hypothesis has to both agree and disagree with $\hat{h}$. Hence, by definition, the \disdis objective is unstable in the overlapping region, where both agreement and disagreement can result in the same discrepancy value. This presents a problem during practical optimization and can result in over-conservative critics which disagree with $\hat{h}$ more than necessary. 

\begin{figure}[t]
    \centering
    \includegraphics[width=0.9\linewidth]{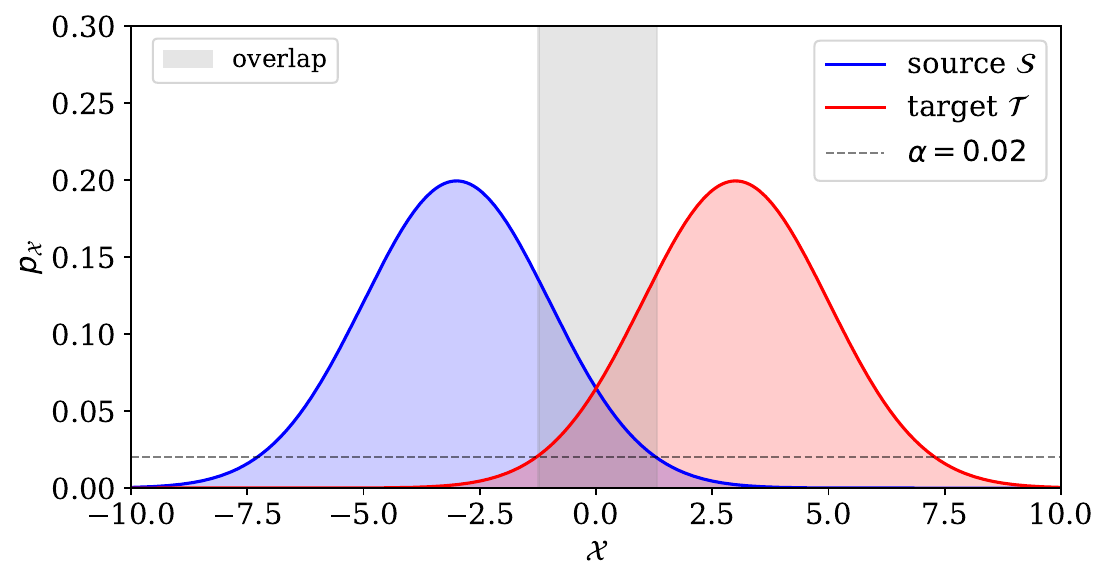}
    \caption{For source ($\gS$) and target ($\gT$) domains with $p_\gS = \mathcal{N}(-3, 2)$ and $p_\gT = \mathcal{N}(3, 2)$ respectively, the domain overlapping region using $\alpha = 0.02$ is shown.}
    \label{fig:overlap}
\end{figure}

\textbf{Domain overlap:} To tackle this problem, we first define domain overlap using a parameter $\alpha$, which quantifies the minimum density requirement in 
both domains of interest for an element to be considered inside the overlapping region. 

\begin{definition}
    \label{def:overlap} The \textit{domain overlap} set between $\gS$ and $\gT$ 
    is defined as $\gD_\alpha = \{x \in \gX: p_\gS(x) > \alpha, p_\gT(x) > \alpha\}$ for some $\alpha > 0$.
\end{definition}
Probability distributions (like the normal distribution), may have non-zero support everywhere, and $\alpha = 0$ makes the entire domain space part of the overlapping region. Practically, an overlap is meaningful only in a region with reasonably high support in domains of interest (see~\autoref{fig:overlap}). 

Now, let $\gK$ denote $[\mathbf{1}\{\argmax_y h(x)_y \neq \argmax_y h'(x)_y\}]$. Then, $\epsilon_{\gD}(h, h') := \E_{x \sim p_\gD}\gK = \int_{x \in \gD}p_\gD(x)\gK dx$. This explicit definition of expectation as an integral allows us to think of the density $p_\gD$ and domain of application $\gD$ independently. If the domain $\gD$ is split into two subsets, say $\gA$ and $\gB$ such that $\gA \cap \gB = \phi$ and $\gA \cup \gB = \gD$, then $\epsilon_{\gD}(h, h') = \int_{x \in \gA}p_\gD(x)\gK dx + \int_{x \in \gB}p_\gD(x)\gK dx$.
We use this separability under the same density as a property to define $\epsilon^\gD_\gA(h, h') = \int_{x \in \gA}p_\gD(x)\gK dx$. Using this definition, we have:
\begin{align*}
    \epsilon_{\gT}(h, h') := \underbrace{\epsilon^\gT_{\gD_\alpha}(h, h')}_\text{Overlap Disagreement} + \underbrace{\epsilon^\gT_{\gT \setminus \gD_\alpha}(h, h')}_\text{Non-Overlap Disagreement}
\end{align*}

This breakdown of disagreement in the overlapping and non-overlapping portions allows us rethink the \disdis objective.

\textbf{Overlap-aware Disagreement Discrepancy (\odd):}
As $\hat{h}$ is trained to minimize the risk in $\gS$, its disagreement with $y^*$ (=$y^*_\gS$) in regions with high support is expected to be low. This is why the \disdis objective aims to minimize the disagreement between the critic $\Bar{h}$ and $\hat{h}$ in the source domain. However, \disdis also tries to maximize the disagreement in $\gD_\alpha$, which is a part of the same high source support region. We address this issue by making \disdis overlap aware. 

\begin{definition}
\label{def:odd}
    The \textit{\textbf{O}verlap-aware \textbf{D}isagreement \textbf{D}iscrepancy (\odd)}
    is defined as the disagreement between any $h$ and $h'$ in the \textit{non-overlapping} region of $\gT$ minus their disagreement in the \textit{non-overlapping} region of $\gS$:
    \begin{align}
    \label{eq:odd}
        \Delta(h, h', \alpha) &:= 
        \epsilon^\gT_{\gT \setminus \gD_\alpha}(h, h') - \epsilon^\gS_{\gS \setminus \gD_\alpha}(h, h').
    \end{align}
\end{definition}

We also have the corresponding \textit{overlap discrepancy}:

\begin{definition}
\label{def:overlap_dis}
$\underline{\Delta}(h, h', \alpha) = \epsilon^\gT_{\gD_\alpha}(h, h') - \epsilon^\gS_{\gD_\alpha}(h, h')$.
\end{definition}


Note that, $\Delta(h, h')$ (\disdis) = $\Delta(h, h', \alpha)$ (\odd) $+ \underline{\Delta}(h, h', \alpha)$. 
With our separation of overlapping and non-overlapping discrepancies, we can rewrite~\autoref{thm:dis2bound} as:

\begin{theorem}
    \label{thm:splitbound}
    Under~\autoref{ass:dis2concave}, 
    with probability $\geq 1-\delta$,
    \begin{align*}
        \epsilon_{\gT}(\hat{h}) &\leq \epsilon_{\hat{\gS}}(\hat{h}) + \hat\Delta(\hat{h}, \Bar{h}, \alpha) + \underline{\hat\Delta}(\hat{h}, \Bar{h}, \alpha) + \\
        &\quad\sqrt{\frac{(n_S + 4n_T) \log \frac{1}{\delta}}{2 n_S n_T}} + \lambda.
    \end{align*}
\end{theorem}

Remember, $\bar{h}$ is the critic from~\autoref{ass:dis2concave}. This is the most general form of our analysis, where terms with $(\hspace{1pt}\hat{}\hspace{1pt})$ denote the empirical counterparts of their corresponding population terms. As we increase $\alpha$, we decrease the size of $\gD_\alpha$. Starting from $\alpha = 0$ when $ \gD_\alpha = \gS \cup \gT$, until we reach some high $\alpha = \alpha_0$ for which $\gD_\alpha = \phi$. The notion of overlap becomes interesting for some intermediate $\alpha$ values, for which samples from either domain behave similarly on learning algorithms. Next, we discuss why the $\underline{\Delta}$ term is expected to be small due to domain overlap.

\textbf{Overlap Discrepancy between $\hat{h}$ and $y^*$:} Of all $h\in\gH$, we usually only care to examine ones which are trained to achieve low source risk, i.e. $\hat{h}$ typically has low disagreement with $y^*$ in regions of high support. For a reasonably chosen $\gD_\alpha$, we have similarly high support from the target domain. A classifier trained with ground truth target labels should be expected to have similar disagreement in this region. In fact, there is no reason to expect $\epsilon^\gT_{\gD_\alpha}(\hat{h}, y^*)$ to be any higher than $\epsilon^\gS_{\gD_\alpha}(\hat{h}, y^*)$. Hence, we make the following assumption:


\begin{assumption}
\label{ass:overlap}
    Target disagreement between $\hat{h}$ and $y^*$ is bounded by the source disagreement in the region of domain overlap, i.e., $\epsilon^\gT_{\gD_\alpha}(\hat{h}, y^*) \leq \epsilon^\gS_{\gD_\alpha}(\hat{h}, y^*)$.
\end{assumption}

In our experiments with randomly generated datasets(~\autoref{fig:synthetic_main}), we show how the overlap discrepancy $(\epsilon^\gT_{\gD_\alpha}(\hat{h}, y^*) - \epsilon^\gS_{\gD_\alpha}(\hat{h}, y^*))$ consistently remains close to zero, supporting our intuition. With~\autoref{ass:overlap} and a common $y^*$ across domains we can derive a practically tighter bound by maximizing \odd. We change~\autoref{ass:dis2concave} to apply only in the non-overlapping region:

\begin{assumption}
    \label{ass:oddconcave}
    For an $\Bar{h}\in\gH$ which maximizes a concave surrogate to the empirical \odd, $\Delta(\hat{h}, y^*, \alpha) \leq \Delta(\hat{h}, \Bar{h}, \alpha)$.
\end{assumption}
Hence,~\autoref{thm:splitbound} is simplified to: 

\fbox{\parbox{\linewidth}{\begin{theorem}[\odd Bound]
    \label{thm:oddbound}
    Under~\autoref{ass:oddconcave} and~\autoref{ass:overlap},
    with probability $\geq 1-\delta$,
    \begin{align*}
        \epsilon_{\gT}(\hat{h}) &\leq \epsilon_{\hat{\gS}}(\hat{h}) + \hat\Delta(\hat{h}, \Bar{h}, \alpha) +
        \sqrt{\frac{(n_S + 4n_T) \log \frac{1}{\delta}}{2 n_S n_T}}
    \end{align*}
\end{theorem}
}}

We discuss this derivation in~\autoref{app:proofs}. Note that this bound does not theoretically improve upon the \disdis based bound from~\citep{rosenfeld2023almost} for a given critic $\Bar{h}$. However, optimizing for \odd allows for the selection of a better critic, as we will see in the next section, which agrees more with $\hat{h}$ in the overlapping source domain. This better critic results in the \odd term being lower than the original \disdis term, making the bound practically tighter.



\section{Maximizing \odd}
\label{sec:algo}
\textbf{Competition in the overlapping region:} The \textit{disagreement logistic loss} of a classifier $h$ on a labelled sample $(x, y)$ is described in \citep{rosenfeld2023almost} as:

\resizebox{\linewidth}{!}{%
\begin{minipage}{\linewidth}
\begin{align*}
    \disloss(h, x, y) := \frac{1}{\log 2} \log\left( 1 + \exp\left( h(x)_y
    - \frac{1}{|\gY|-1} \sum_{\hat y\neq y} h(x)_{\hat y} \right)\right).
\end{align*}
\end{minipage}%
}

$\disloss$ is shown to be convex in $h(x)$ and to upper bound the 0-1 disagreement loss. To maximize the empirical disagreement discrepancy, a combined loss on source and target samples is used:
\begin{align*}
    \hat\gL_\Delta(\Bar{h}) := &\ \frac{1}{|\hat{\gS}|} \sum_{x\in\hat{\gS}} \logloss(\Bar{h}, x, \hat{h}(x)) \\
    &+ \frac{1}{|\hat\gT|} \sum_{x\in\hat\gT} \disloss(\Bar{h}, x, \hat{h}(x)).
\end{align*}

Here, $\logloss := -\frac{1}{\log |\gY|} \log \textrm{softmax}(h(x))_y$ is the standard log-loss. This is where we see the competition: for points that lie in the overlapping region, one would contribute by reducing $\logloss$, making $\Bar{h}$ agree with $\hat{h}$ while the other would contribute by reducing $\disloss$, making $\Bar{h}$ disagree with $\hat{h}$. In practice, this typically results in an $\Bar{h}$ that agrees with $\hat{h}$ a little less than it could have in the overlapping region. We illustrate this phenomena in~\autoref{fig:synthetic_main} and show how our proposed method improves on it, which we describe next. 

\textbf{Mitigating the competition:} We propose to improve this loss by discounting the disagreement of target samples in the overlapping region, i.e.,
\begin{align*}
    \hat\gL_{\Delta}(\Bar{h}, \alpha) := &\ \frac{1}{|\hat{\gS}|} \sum_{x\in\hat{\gS}} \logloss(\Bar{h}, x, \hat{h}(x)) \\
    &+ \frac{1}{|\hat\gT|} \sum_{x\in\hat\gT} \mathbf{1}\left\{x \notin (\hat{\gD_\alpha})\right\}*\disloss(\Bar{h}, x, \hat{h}(x)).
\end{align*}

Here, $\mathbf{1}\left\{x \notin (\hat{\gD_\alpha})\right\}$ is an indicator function which only selects samples that are outside the overlapping set. For samples in the overlapping set, the $\disloss$ term is discounted to zero, encouraging the maximizing $\Bar{h}$ to not disagree with $\hat{h}$ in this region. Note that this loss still maximizes the empirical \odd, as the disagreement in the non-overlapping target region is still maximized and the disagreement in the non-overlapping source region is still minimized. 

\textbf{How to select $\gD_\alpha$?} The ideal way to find the overlap set $\gD_\alpha$ is to find the densities $p_\gS$ and $p_\gT$, and tune for the smallest $\alpha$ on a held-out validation set which satisfies~\autoref{ass:overlap}. However, density estimation is typically difficult with high-dimensional data and accessing a held-out set may be challenging in many domains where data is scarce (where reliable performance estimation is the most essential). We propose to do it with a simple method of domain classification. If there is overlap between the source and target domains, a domain classifier should have difficulty in classifying samples from this overlapping region. As we already know domain labels, we use the domain classifier's prediction as a proxy to quantify the extent of overlap. 

Suppose a domain classifier $d: \gX \rightarrow \{0, 1\}$ is trained to classify all source samples with the label $0$ and all target samples with the label $1$. It produces logits $\{d(x)_0, d(x)_1\}$ for all samples and taking the $\argmax$ gives us the classified domain. We can replace the indicator function $\mathbf{1}\left\{x \notin (\hat{\gD_\alpha})\right\}$ with $\mathbf{1}\left\{\argmax(d(x)) \neq 1)\right\}$ to find target domain samples which are being classified as source samples, and hence have an estimate of the set $\gD_\alpha$. Note that this assumes our domain classifier model is accurate.  
In practice, we can convert this problem of hard selection of the set $\gD_\alpha$ into a soft version, by simply weighting each sample in the target domain by its probability of being classified in the target domain. Let $\{s(x)_0, s(x)_1\}$ denote the soft-max of $\{d(x)_0, d(x)_1\}$. Then, the indicator function $\mathbf{1}\left\{x \notin (\hat{\gD_\alpha})\right\}$ can be replaced with $s(x)_1$. Although the soft-max probabilities are not guaranteed to be calibrated without the use of a held-out calibration set, this soft version makes the discounting of overlapping target samples smooth and it works well in practice. Hence, we use the following loss to maximize \odd while making $\Bar{h}$ agree with $\hat{h}$ in the overlapping region:
\begin{align*}
    \hat\gL_{\Delta}(\Bar{h}, \alpha) := &\ \frac{1}{|\hat{\gS}|} \sum_{x\in\hat{\gS}} \logloss(\Bar{h}, x, \hat{h}(x)) \\
    &+ \frac{1}{|\hat\gT|} \sum_{x\in\hat\gT} s(x)_1 *\disloss(\Bar{h}, x, \hat{h}(x)).
\end{align*}
We describe our method step by step in Algorithm~\ref{algo:odd_weight}.

\textbf{Why not also discount source overlapping samples?} 
We want $\Bar{h}$ to agree with $\hat{h}$ in this region, which includes both source and target samples. Discounting loss from both overlapping source samples and overlapping target samples would lead to a problem similar to that of $\hat\gL_\Delta(\Bar{h})$, where instead of competition, the optimization would be indifferent to all the samples in the overlapping region. In~\autoref{app:expdetails}, we conduct experiments to validate this and discuss how this may lead to worse critics.

\begin{algorithm}[tb]
\caption{Finding the critic $\Bar{h}$ for a given $\hat{h}$}
\label{algo:odd_weight}
\begin{algorithmic}[0]
    \State {\bfseries Input:} $\hat{h}, \hat{\gS} = \{(x^\gS_i, y^*_i)\}_{i=1}^{n_\gS}, \hat{\gT} = \{(x^\gT_i)\}_{i=1}^{n_\gT}$
    \State $d \gets \texttt{train}(\{(x^\gS_i, 0)\}_{i=1}^{n_\gS} \cup \{(x^\gT_i, 1)\}_{i=1}^{n_\gT})$ \\\Comment{Train domain classifier.}
    \State $s \gets \texttt{softmax}(d)$
    \State $p^{\hat{h}} \gets \argmax \hat{h}$ \Comment{Returns the $\hat{h}$ predicted label.}
    \State $\Bar{h} \gets \hat{h}$ \Comment{Initialize critic with $\hat{h}$ parameters.}
    \State $\Theta \in \Bar{h} \gets \text{require grad}$ \Comment{Choose parameters to learn.}
    \While{not converged}
        \State $\gS_{\text{loss}} \gets \texttt{mean}(\{\logloss((x^\gS_i, p^{\hat{h}}(x^\gS_i)))\}_{i=1}^{n_\gS})$
        \State $\gT_{\text{loss}} \gets \texttt{mean}(\{s(x^\gT_i)_1 * \disloss((x^\gT_i, p^{\hat{h}}(x^\gT_i)))\}_{i=1}^{n_\gT})$
        \State $\Theta'_t \gets \Theta'_{t-1} - \frac{\partial (\gS_{\text{loss}} + \gT_{\text{loss}})}{\partial \Theta}$ \Comment{Update $\Bar{h}$ parameters.}
    \EndWhile
    \State return $\Bar{h}$
\end{algorithmic}
\end{algorithm}

\section{Experimental Results}
\label{sec:exp}

\subsection{Synthetic data experiments}
\label{subsec:synthetic_exps}
\textbf{Dataset:} To visually illustrate the benefit of our overlap-aware disagreement discrepancy, we conduct experiments on 2-D synthetic data. First, we sample source and target domains from 2 randomly initialized Gaussian distributions with a randomly sampled overlap factor, which determines how close their means are. Then, we initialize a complex decision surface $y^*$ which determines the labels for each sample (with some noise). One sample dataset can be seen in~\autoref{fig:synthetic}. Then, we train a model $\hat{h}$ on the source data. Finally, we use Algorithm~\ref{algo:odd_weight} to find $\Bar{h}$ to obtain a bound on the target performance (accuracy). For comparison and analysis, we also use \disdis to obtain the same bound. The details of the dataset generation can be found in~\autoref{app:expdetails}.

\begin{figure}[t]
    \centering
    \includegraphics[width=0.85\linewidth]{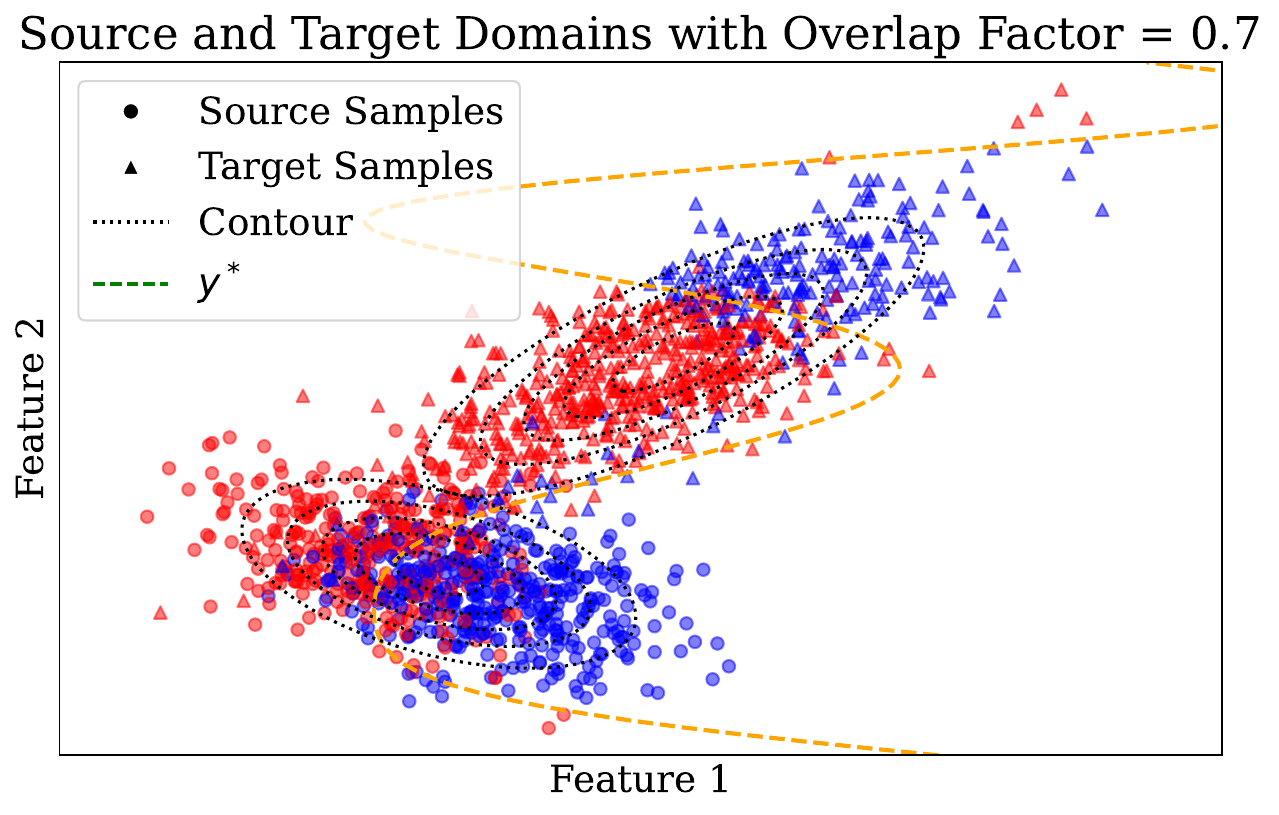}
    \caption{Sample 2-D synthetic dataset used for analysis. With a complex $y^*$, the model learned on source data can make errors in the target domain, but the overlapping region has high agreement.}
    \label{fig:synthetic}
\end{figure}

\textbf{Methodology:} We randomly sample an overlap rate $100$ times between $[0, 1]$ and generate a new dataset for each (around $2000$ training samples and $1250$ validation samples in each domain). We repeat the experiment $40$ times, effectively creating $4000$ unique datasets, to smoothen any finite sampling issues. We use a small $3-$layer, $16-$neurons wide MLP to train the $\hat{h}, \Bar{h}$ and $d$. Finally, we distribute the $4000$ overlap rates into $20$ equal width bins to average our findings across multiple datasets. Instead of predicting the error bound (as in~\autoref{thm:oddbound}), we predict the accuracy bound to follow the convention in~\citet{rosenfeld2023almost}. Hence, a tighter bound implies a larger lower bound in accuracy.

\begin{figure}[t]
    \centering
    \includegraphics[width=\linewidth]{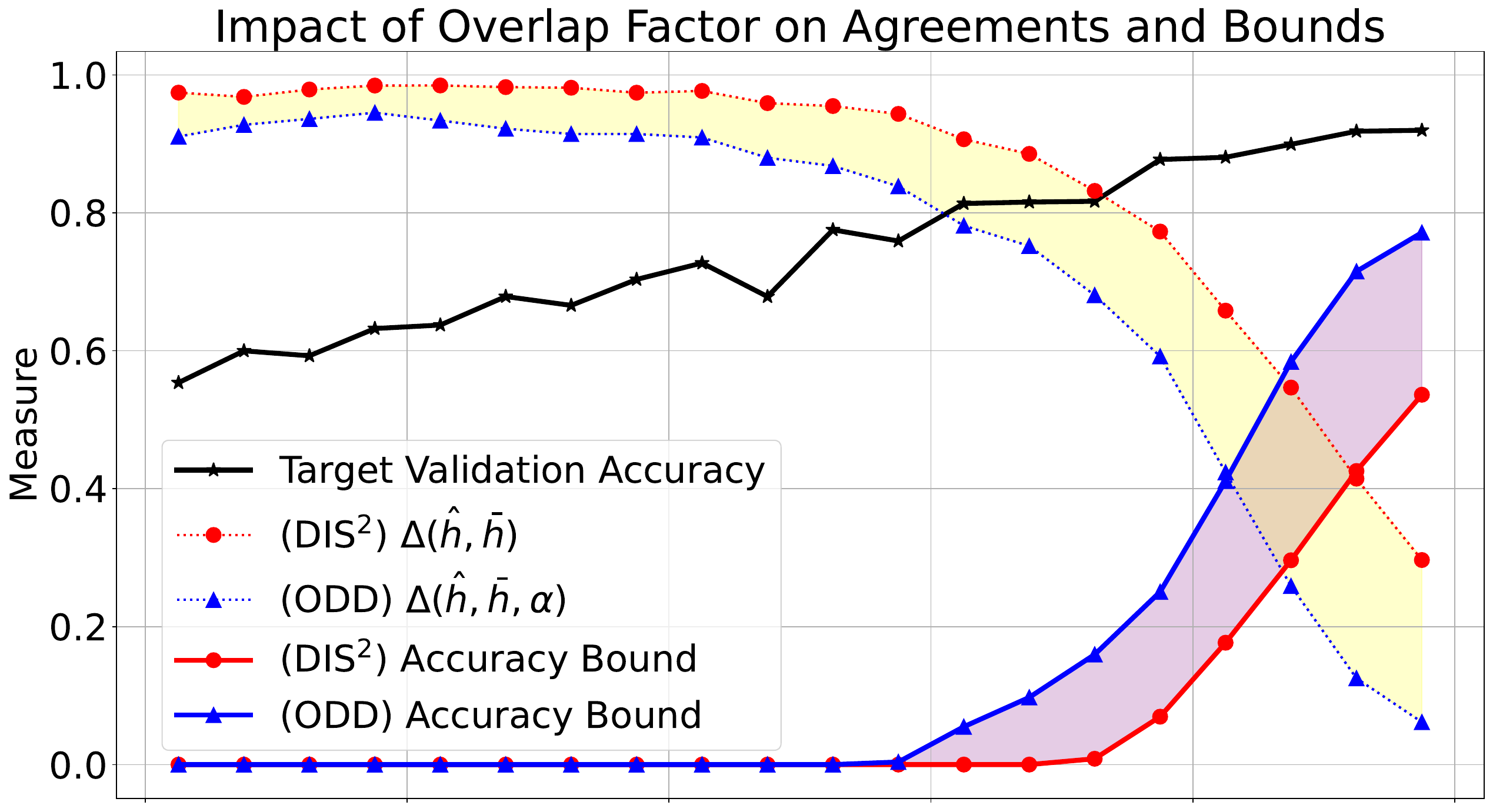}\\
    \includegraphics[width=\linewidth]{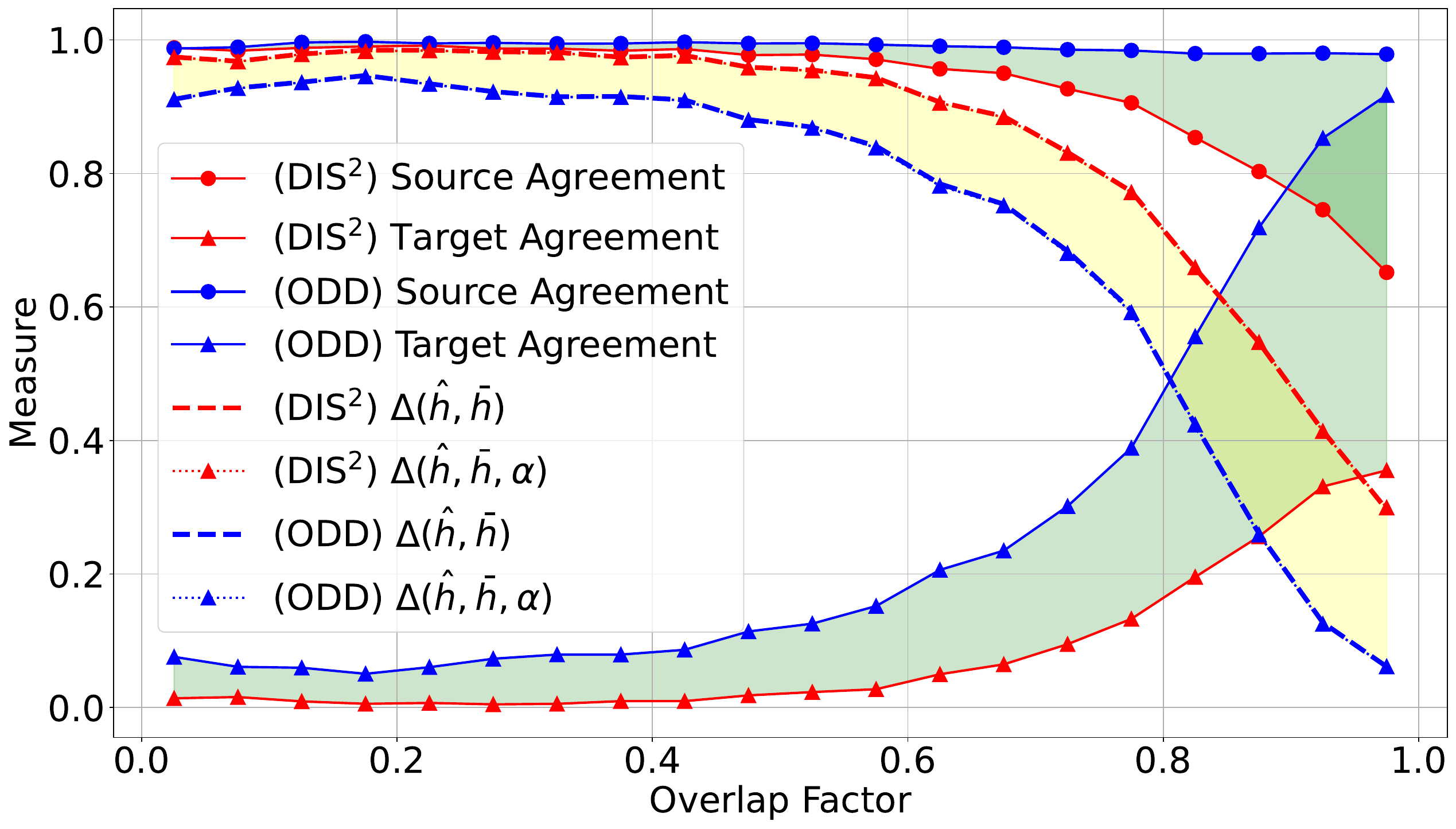}
    \caption{\textbf{TOP:} As the overlap factor increases, \odd becomes increasingly smaller than \disdis (yellow shaded region), resulting in a tighter bound (purple shaded region) that is closer to the target accuracy. 
    \textbf{BOTTOM:} 
    We decompose the discrepancy into individual source and target agreements. 
    Due to the competition in the overlapping region, maximizing \disdis makes $\Bar{h}$ start disagreeing with $\hat{h}$ in the source domain, while maximizing \odd maintains agreement. \odd also allows for much higher target agreement in high overlap cases. Both shown with green shaded region. This results in a less conservative $\Bar{h}$. Moreover, we find that $\Delta(\hat{h}, \Bar{h})\approx\Delta(\hat{h}, \Bar{h}, \alpha)$ for both \odd and \disdis based critics (difference $< 0.01\%$ making plot-lines coincide in this figure), as predicted by our theory.}
    \label{fig:synthetic_main}
\end{figure}

\textbf{Findings---why \odd is better than \disdis:}
In~\autoref{fig:synthetic_main}, we show the relationship between various metrics (discrepancy, target accuracy, predicted target accuracy (the bound) and individual domain agreements) with the overlap rate (for both \odd and \disdis). The target performance prediction is calculated as $\epsilon_\gS(\hat{h}, y^*) - \Delta(\hat{h}, \Bar{h}) - \varepsilon$ for \disdis, and $\epsilon_\gS(\hat{h}, y^*) - \Delta(\hat{h}, \Bar{h}, \alpha) - \varepsilon$ for \odd ($\varepsilon$ is the concentration term in~\autoref{thm:oddbound}). We find that:
\begin{itemize}
    \item As the overlap increases, \disdis-based critics start disagreeing with $\hat{h}$ in the source domain, while \odd-based critics maintain their source agreement.
    \item Even with high overlap, \disdis-based critics disagree substantially with $\hat{h}$ in the target domain, making them overly pessimistic. \odd fixes this by increasing target agreement in high overlap regions. 
    \item For both \disdis and \odd based critics, overall discrepancy $\Delta(\hat{h}, \bar{h})$ is approximately equal to non-overlapping discrepancy $\Delta(\hat{h}, \Bar{h}, \alpha)$, validating our assumption and theory in~\autoref{sec:odd}. This shows that \textit{\odd improves the bound by selecting a better critic}, not by reducing the disagreement for a given critic. 
\end{itemize}

\textit{\odd consistently improves over \disdis by choosing a more optimistic $\Bar{h}$ that agrees more with $\hat{h}$ in the overlapping region, while remaining valid and reliable.}
It makes sure that the critic is chosen as optimistically as possible based on the overlap, without compromising disagreement in the unseen target region.

\begin{table*}[t]
    \begin{adjustbox}{width=0.68\linewidth,center}
    \centering
    \tabcolsep=0.12cm
    \renewcommand{\arraystretch}{1.2}
    \begin{tabular}{lccccccccc}
    \toprule
    {} && \multicolumn{2}{c}{MAE $(\downarrow)$} && \multicolumn{2}{c}{Coverage $(\uparrow)$} && \multicolumn{2}{c}{Overest. $(\downarrow)$} \\
    \multicolumn{1}{r}{\textbf{DA?}} &&          \ding{55} & \ding{51} &&             \ding{55} & \ding{51} &&               \ding{55} & \ding{51} \\
    Prediction Method             &&                    &           &&                      &           &&                         &           \\
    \midrule
    AC \citep{guo2017calibration}                            &&            0.1086 &    0.1091 &&                0.0989 &    0.0333 &&                  0.1187 &    0.1123 \\
    DoC \citep{guillory2021predicting}                           &&             0.1052 &    0.1083 &&                0.1648 &    0.0167 &&                  0.1230 &    0.1095 \\
    ATC NE \citep{garg2022leveraging}                        &&             0.0663 &    0.0830 &&                0.2857 &    0.2000 &&                  0.0820 &    0.1007 \\
    COT \citep{lu2023predicting}                           &&             0.0695 &    0.0808 &&                0.2528 &    0.1833 &&                  0.0858 &    0.0967 \\
    \midrule
    \disdis~\citep{rosenfeld2023almost} && & && & && &\\ 
     \quad Using Features            &&             0.2841 &    0.1886 &&                1.0000 &    1.0000 &&                  0.0000 &    0.0000 \\
     \quad Using Logits              &&             0.1525 &    0.0881 &&                0.9890 &    0.7500 &&                  0.0171 &    0.0497 \\
    \quad Using Logits w/o $\delta$ term &&             0.0996 &    0.0937 &&                0.6813 &    0.2833 &&                  0.0779 &    0.0956 \\
    \midrule
    \odd && & && & && &\\ 
     \quad Using Features            &&             0.2538 &    0.1657 &&                0.9890 &    0.9333 &&                  0.0314 &    0.0318 \\
     \quad Using Logits              &&             0.1190 &    0.0739 &&                0.9341 &    0.7333 &&                 0.0290 &    0.0738 \\
    \quad Using Logits w/o $\delta$ term &&             0.0943 &    0.1005 &&                0.4945 &    0.2000 &&                  0.0803 &    0.1079 \\
    \bottomrule

    \end{tabular}
    \end{adjustbox}
    \caption{Comparing the \odd bound with \disdis bound and other prior methods for predicting accuracy. DA denotes if the representations were learned via a domain-adversarial algorithm (DANN, CDANN). MAE: mean absolute error, Coverage: fraction of predictions correctly bounding the true error, Overest.: MAE among shifts whose accuracy is overestimated. \odd improves on MAE in comparison to \disdis with some loss in coverage, but maintains a much higher lead in coverage over prior methods.}
    \label{tab:main}
\end{table*}

\subsection{Real data experiments}
\label{subsec:real_exps}
\textbf{Datasets:} As we aim to achieve an improvement of \disdis~\citep{rosenfeld2023almost}, we follow their setup and conduct experiments across all vision benchmark datasets used in their study for a fair comparison between \disdis and \odd. These include four BREEDs datasets ~\citep{santurkar2020breeds}:
{Entity13}, {Entity30}, {Nonliving26}, and {Living17}; {FMoW}~\citep{christie2018functional}
from {WILDS}~\citep{wilds2021}; Officehome~\citep{venkateswara2017deep}; {Visda}~\citep{peng2018syn2real, visda2017}; CIFAR10, CIFAR100~\citep{krizhevsky2009learning}; and Domainet~\citep{peng2019moment}. In addition, we conduct additional experiments on a language dataset: CivilComments~\citep{borkan2019nuanced} for breadth of applicable domains. These datasets contain multiple domains and include a wide variety of subpopulation and natural distribution shifts. We present some details about these datasets and their shift variations, and complete details about the new CivilComments dataset in~\autoref{app:expdetails}. Comprehensive details can be found in the \disdis paper.

\textbf{Methods and baselines:} Models are trained with ERM and Unsupervised Domain Adaptation methods (which help improve target performance with unlabeled target data) like FixMatch~\citep{sohn2020fixmatch}, DANN~\citep{ganin2016domain}, CDAN~\citep{long2018conditional}, and
BN-adapt~\citep{li2016revisiting}). Although our paper is positioned as an improvement to \disdis, we also present comparisons with other baselines like Average Confidence (AC)~\citep{guo2017calibration}, Difference of Confidences (DoC)~\citep{guillory2021predicting}, Average Thresholded Confidence (ATC)~\citep{garg2022leveraging}, and Confidence Optimal Transport (COT)~\citep{lu2023predicting}, for completeness. All methods are calibrated with temperature scaling~\citep{guo2017calibration} using source validation data. As our general setup is identical, we follow the implementation details in the \disdis paper. We list experimental details on training the domain classifiers in ~\autoref{app:expdetails}. For the correction term in~\autoref{thm:oddbound}, we use a small $\delta=0.01$ (same as \disdis) in all our experiments. 

\textbf{Bound calculation strategies:} As pointed out by~\citet{rosenfeld2023almost}, the value of the error bound is expected to decrease if the critic is searched for in a restricted hypothesis class. Their real-data experiments are performed on the deep features of the inputs, as classifiers trained on these features have been shown to have the capacity to generalize under distribution shifts~\citep{rosenfeld2022domain, kirichenko2022last}. This makes $\hat{h}$ belong to the linear hypothesis class. Even logits of the source classifier may contain information necessary for this task as prior work suggests that deep network representations have small effective ranks~\citep{arora2018optimization, huh2022lowrank, pezeshki2021gradient}. Therefore, we follow their setup to calculate our bounds using both: the deep feature space, and the logit space. Our domain classifier (used for measuring domain overlap) is also trained on the same space (details in~\autoref{app:expdetails}). Lastly, the $\delta$ term in~\autoref{thm:oddbound} may make the bound too conservative in practice, so we also calculate the bound without it for a complete comparison with \disdis.

\textbf{Evaluation metrics:} We evaluate all methods on mean absolute error (MAE), which measures how close the predicted performance is to the actual performance. In addition, we report coverage, which is the fraction of accuracy predictions that are valid ($\leq$ true accuracies). This measures the reliability of the method. We also report the MAE among predictions which are invalid, to get a sense of how bad the over-estimates are, when they do happen. 

\textbf{Domain-adversarially learnt representations:}
\label{para:da} Domain adversarial representation learning methods like DANN and CDANN, regularize deep representations to be indifferent between source and target domains, to make classifiers robust to distribution shifts. This regularization may cause representations to lose domain specific features resulting in a higher degree of overlap between source and target domain features. A higher degree of overlap typically means a more lenient bound using our \odd method, as it suggests that the two domains are similar (see~\autoref{fig:synthetic_main}). Therefore, when the method produces a classifier which actually has low target error, the \odd-based bound will be tighter. However, when the actual target error is high, it means that the representations have lost important information from the target domain and the high overlap hurts, causing the \odd-based bound to overestimate accuracy. Therefore, we present our findings in two categories: \textbf{DA?} \ding{51}: representing the case when the representations were learnt using a domain adversarial algorithm, and \textbf{DA?} \ding{55} otherwise.

\textbf{Results:} We report our comprehensive evaluation metrics in~\autoref{tab:main}. We find that \textit{\odd-based bounds are more accurate than \disdis-based bounds} in general; and significantly improve the estimate in many cases (especially with non-DA methods). Moreover, this improved MAE does not come at a huge cost of coverage, where the \odd-based bound maintains a much higher coverage compared to other baselines. We discuss non-DA results in~\autoref{app:expdetails}.

\textbf{Improvements in Valid vs Invalid predictions:}
In~\autoref{tab:invalidity}, we segregate the predictions on the basis of validity to investigate failure cases. In most cases, both \disdis and \odd based bounds are valid and, the \odd-based bound is significantly more tight compared to \disdis. In a small number of cases, \odd overshoots the target accuracy while \disdis remains valid. Even in these cases, the MAE of \odd is better than \disdis, meaning it only slightly overestimates the performance of $\hat{h}$. And in a few cases, \odd achieves a valid bound when \disdis does not. We conducted a paired Student’s t-test on the distributions of (target accuracy $-$ predicted lower bound) for \odd and \disdis (for non-DA algorithms). Out of 94 total predictions, 89 were valid (predicted lower bound was actually lower than target accuracy) for both \odd and \disdis. To remove any influence of invalid predictions, we show the t-statistic for both valid and all predictions cases:

\textit{Valid predictions}: t-statistic: \textbf{-5.47}, p-value: 4.14e-07\\
\textit{All predictions}: t-statistic: \textbf{-5.62}, p-value: 2.02e-07

As evident, we achieve a high t-statistic with a very low p-value, giving strong evidence against the null-hypothesis (i.e., no significant difference between the two methods). 

\textbf{Variation due to degree of overlap:} We plot point-wise bound estimates (using logits) for all non-DA trained $\hat{h}$ (~\autoref{fig:real_logits}) for more analysis. 
When actual target accuracy is low (left part of the plot), the overlap between source and target domains is likely low. Hence, we see little improvement over \disdis in this region. When actual target accuracy is high (right part of the plot), domain overlap is likely also high. Therefore, we see a lot more improvement in this region. Around the middle is when overlap estimate is likely the hardest. This is where \odd sometimes overestimates the overlap and the target accuracy (only slightly) as a result.

\begin{table}[t]
    \begin{adjustbox}{width=0.95\linewidth,center}
    \centering
    \tabcolsep=0.12cm
    \renewcommand{\arraystretch}{1.2}
    \begin{tabular}{lccccccccc}
    \toprule
    {MAE $(\downarrow)$} && \multicolumn{2}{c}{\textbf{DA:} \ding{55}} && \multicolumn{2}{c}{\textbf{DA:} \ding{51}}\\
    \multicolumn{1}{r}{} &&          \disdis & \odd &&          \disdis & \odd\\
    Prediction set             &&                    & && &          \\
    \midrule
     \disdis invalid, \odd valid            &&             0.0171 &   \textbf{0.0029}  && 0.0489 &    \textbf{0.0032}\\
     \disdis valid, \odd invalid              &&             0.0540 &  \textbf{0.0248}   &&       0.0325       &    \textbf{0.0215}\\
    Both valid &&             0.1611 &   \textbf{0.1234}  &&             0.1058 &    \textbf{0.0772}\\
    Both invalid &&  0.0000 & 0.0000 &&  \textbf{0.0498} & 0.0822 \\
    \bottomrule

    \end{tabular}
    \end{adjustbox}
    \caption{Comparing MAE of \disdis and \odd conditioned on coverage. When \odd-based bound is invalid, it typically overestimates by a small amount. In comparison, \disdis typically has a larger MAE when it is invalid. In the most popular case, when both are valid, \odd consistently outperforms \disdis.}
    \label{tab:invalidity}
\end{table}

\section{Discussion}
\label{sec:discuss}
While \odd shows consistent improvement in prediction accuracies, we discuss some potential limitations and improvements to our method.

\textbf{Estimating overlap:} We find using a small domain-classifier's softmax probabilities as weights to discount the \disdis contribution of overlapping target samples effective in improving the bound estimation accuracy. However, this approach does not guarantee a good quantification of overlap and is affected by factors like the training hyper-parameters of the domain-classifier (underfitting/overfitting) and the representations on which the classifier is being trained (like DA vs non-DA). A true overlap-aware bound would require accurate density estimation of the source and target densities, which is challenging for high-dimensional data. Nonetheless, our positive results with the CivilComments dataset using BERT embeddings paint a promising picture for the future with increasingly better representation learning.

\begin{figure}[t]
    \centering
    \includegraphics[width=\linewidth]{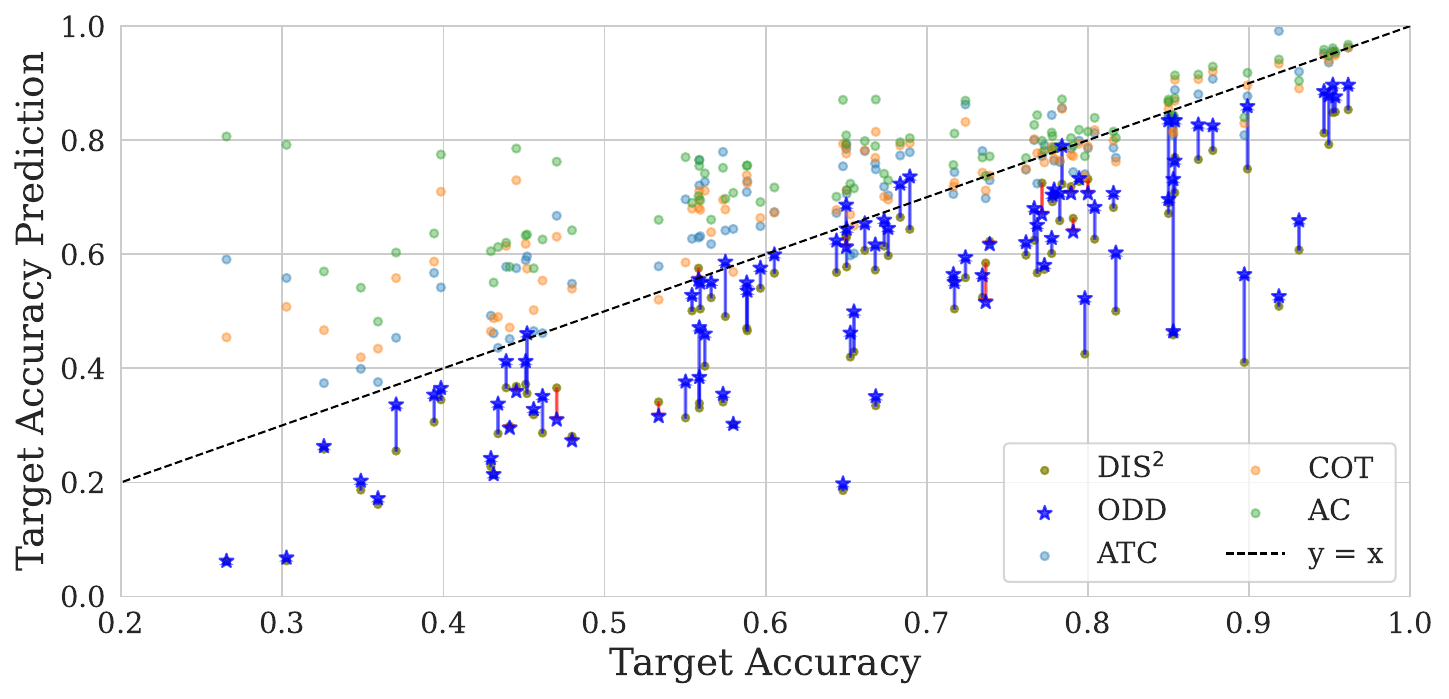}
    \caption{Our \odd-based bound (using logits, \textbf{DA?} \ding{55}) improves the \disdis bound on multiple dataset-method combinations, achieving an improved average MAE. Improvement is shown with \textcolor{blue}{$\uparrow$} (deterioration with \textcolor{red}{$\downarrow$}). Notice that our bound, like \disdis, still almost always maintains its reliability, unlike other baseline methods.}
    \label{fig:real_logits}
\end{figure}

\begin{figure*}[t]
    \centering
    \includegraphics[width = 0.7\linewidth]{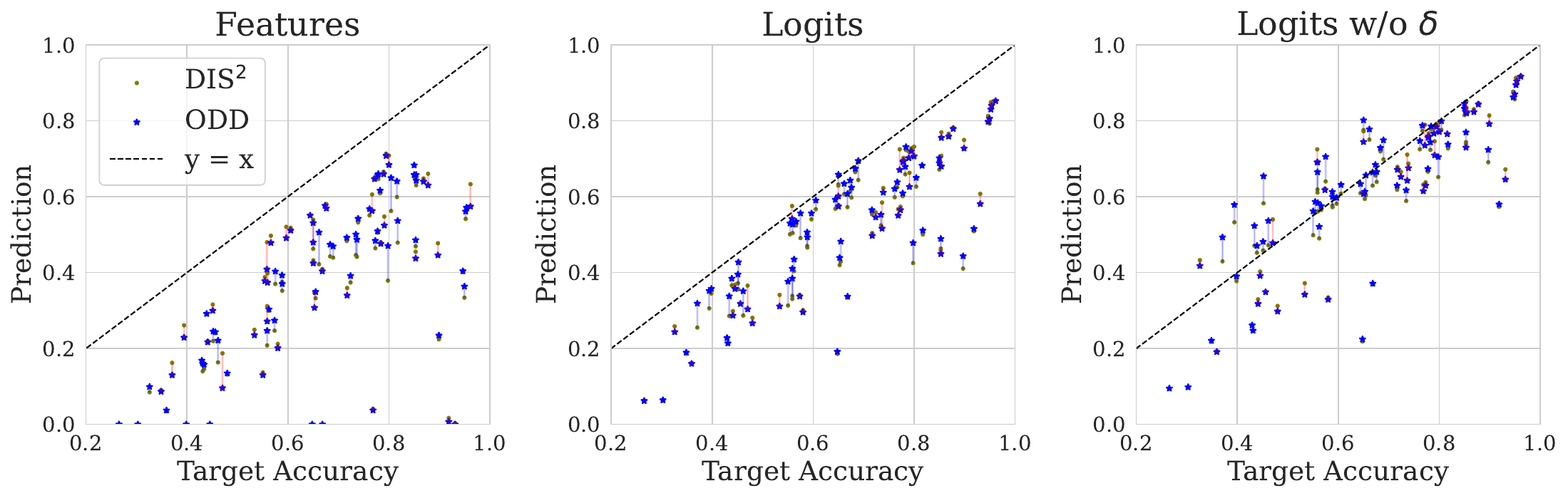}
    \includegraphics[width = 0.25\linewidth]{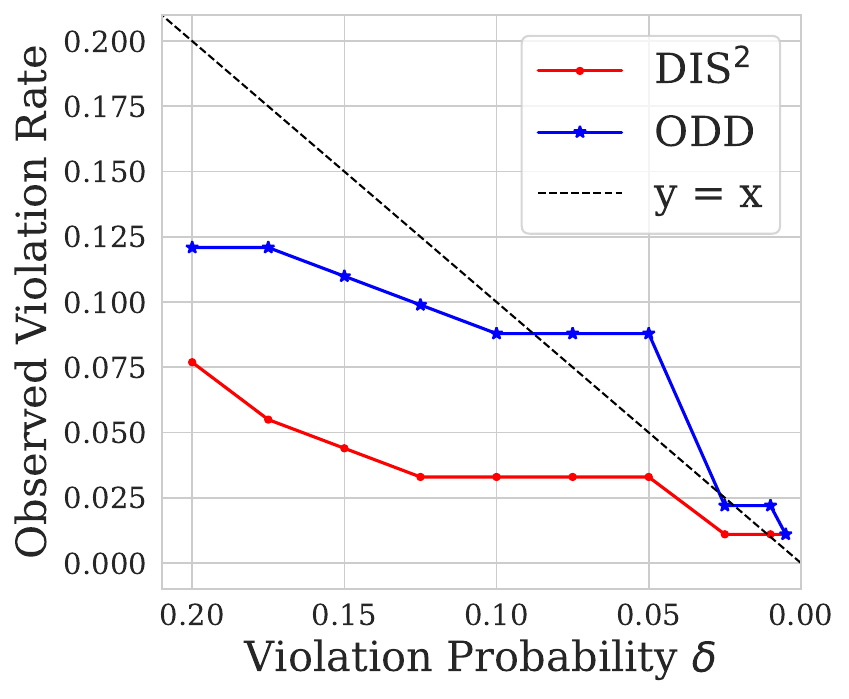}
    \caption{\textbf{(First 3)}: Comparing \disdis and \odd on estimated bound based on features and logits vs true accuracy. ``w/o $\delta$'' drops the $\delta$ term of~\autoref{thm:oddbound}. \textbf{(4$^{\text{th}}$)} Observed bound violation rate vs. desired probability $\delta$ from~\autoref{thm:oddbound} based on logits. The observed violation rate for both \disdis and \odd remains under the allowed violation probability for most $\delta$'s, but \odd overshoots it for a small set.}
    \label{fig:coverage}
\end{figure*}

\textbf{When does \odd fail?} In~\autoref{fig:coverage}, we compare the estimated bound and the true accuracy on a variety of non-DA representations in both the feature and logit space. We find that both \odd and \disdis always remain valid in the case of feature space, and almost always valid in the case of logit space. In both cases, \odd improves on \disdis with a more accurate prediction of the bound. However, without the correction $\delta$ term in~\autoref{thm:oddbound}, both methods overestimate the accuracy many times. We observe that for a small range of $\delta$ values, \odd marginally overshoots the allowed probability of violations. We suspect that this could be caused by an over-estimation of the overlap between the two domains in some cases, making the critic agree more with $\hat{h}$ than it should have. Better overlap estimation should be able to mitigate this issue, which we defer to future work.

\textbf{Sample Complexity:}
 Our theory does not include sample complexity analysis of training the evaluated classifier. We assume that the classifier is given to us and we focus on estimating the target performance with finite samples. We restricted our search for the critic in the linear hypothesis class which has a uniform sample complexity, but sample complexity for linear classifiers can still be measured and analyzed in a non-uniform way. 
In our analysis, following~\citet{rosenfeld2023almost}, we repeated our experiments 30 times to dilute variance (for each dataset, the critic used in bound calculation was chosen from a set of 30 critics learnt on the same data, based on max discrepancy). 
Regarding analyzing the variance and how it can in turn affect our discrepancy measure and the final bound with a non-uniform sample complexity, we refer to future work.

\section{Related Work}
\label{sec:related}

\paragraph{Generalization bounds under distribution shift:} Finding generalization bounds under distribution shift is a classic machine learning problem with a long history of development~\citep{redko2019advances}. Early works using $\gH$ and $\gH\Delta\gH$ divergence~\citep{bendavid2007analysis, mansour2009domain, bendavid2010theory} established domain adaptation bounds through uniform convergence. These works inspired many efforts in making classifiers robust to distribution shifts~\citep{zhang2019shiftinvar,ganin2016domain, long2018conditional,rahimian2019distributionally, sagawa2019distributionally,arjovsky2019invariant}. In particular, a major line of work leverages representation invariance across domains for domain adaptation and generalization \citep{johansson2019support,zhao2019learning}
Theoretically, the $\gH\Delta\gH$ divergence motivates further improvement on the definition of divergence to measure the difference between domains, with respect to a particular function class \citep{zhang2019bridging}. Application to various types of distribution shift \citep{awasthi2023theory} and new methods for the computation of the discrepancy \citep{kuroki2019unsupervised} also followed. 
In addition, PAC-Bayesian analysis has also been applied to this problem \citep{germain2013pac, germain2016new} which was recently extended to provide non-uniform sample complexity analysis~\citep{sicilia2022pac}. However, due to the general difficulty of estimating the divergence, it is usually hard to directly estimate the target performance following theoretical bounds.  \citep{rosenfeld2023almost} developed \disdis as a theoretically sound but practical method to bound the target risk of a given source classifier, which is the setting of focus in this work. 

Another line of work aims to bound the generalization risk of models (espeically neural networks) by measuring and using the complexity of these models~\citep{bartlett2017spectrally, dziugaite2017computing, zhou2018non}. By evaluating the complexity, these works estimate the expressive capacity of the models providing insights into their worst-case behavior, but often lead to loose bounds.

\textbf{Predicting error in a target domain:} Other works focus on predicting the error of a trained classifier/network on unlabeled samples from a target domain. They either provide instance-level estimates using ensembles or data augmentations~\citep{chen2021detecting, deng2021labels} or an overall domain-level error using empirical observations~\citep{baek2022agreement}, domain-invariant representations~\citep{chuang2020estimating}, average thresholded confidence~\citep{garg2022leveraging} or difference of confidences~\citep{guillory2021predicting}. Some of these methods require labeled target samples for calibration, which makes their applicability limited.
In our work, we mainly focus our effort on improving \disdis~\citep{rosenfeld2023almost} as it presents a strong practical method to give reliable performance bounds estimates. Reliability is a major concern especially when the actual target performance is low.
With this work, we push to improve the prediction accuracy of \disdis without compromising reliability.

\section{Conclusion}
\label{sec:conclusion}

In this work, we identified a potential for improvement in recent work which uses disagreement discrepancy (\disdis) to calculate practical bounds on the performance of a source trained classifier in unseen target domains. \disdis suffers from a competition between source and target samples for agreement and disagreement respectively in the overlapping region, which may results in a suboptimal critic due to unstable optimization. We mitigate this issue by using overlap-aware disagreement discrepancy (\odd). By restricting the disagreement in the non-overlapping target domains, we eliminate the instability in the overlapping region, and derive a new bound on the basis of \odd. We train a domain classifier to discriminate between source and target samples. 
We use the softmax probabilities of this domain classifier to act as weights to discount the disagreement of target samples in the overlapping region. This makes the critic agree with source samples in the overlapping region and improves performance bounds as a result, without much compromise in the reliability of the bound. 

In the future, we hope to further improve the overlap estimation step. Moreover, in the era of Large Language Models (LLMs), where the input/output space is highly complex, we aim to extend our method to develop generic definitions of disagreement and overlap beyond classification settings; and provide performance estimation for LLMs.

\begin{acknowledgements} 
We thank the reviewers for their insightful feedback. We also benefit from the highly reproducible work by~\citet{rosenfeld2023almost} in our experiments.
Both authors are partially supported by a seed grant from JHU Institute of Assured Autonomy (IAA). AL is also partially supported by an Amazon Research Award.
\end{acknowledgements}

\bibliography{odd}

\newpage

\onecolumn

\title{ODD: Overlap-aware Estimation of Model Performance under Distribution Shift\\(Supplementary Material)}
\maketitle

\appendix
\section{Proofs}
\label{app:proofs}

\begin{lemma}
\label{lemma:error-expansion_app}
    For any classifier $h$, \\
    $\epsilon_{\gT}(h) \leq \epsilon_{\gS}(h) + \Delta(h, y^*) + \lambda$.
\end{lemma}
\begin{proof}
    \begin{align*}
        \epsilon_{\gT}(h) &\leq \epsilon_{\gT}(h, y^*) + \epsilon_{\gT}(y^*) \quad \text{[Triangle Inequality]}\\
        &= \epsilon_{\gS}(h, y^*) + (\epsilon_{\gT}(h, y^*) - \epsilon_{\gS}(h, y^*)) + \epsilon_{\gT}(y^*)\\
        &\leq \epsilon_{\gS}(h) + \Delta(h, y^*) + (\epsilon_{\gT}(y^*) + \epsilon_{\gS}(y^*)) \\
        &= \epsilon_{\gS}(h) +\Delta(h, y^*) + \lambda
    \end{align*}
\end{proof}

\begin{theorem}[\disdis Bound]
    \label{thm:dis2bound_app}
    Under~\autoref{ass:dis2concave}, with probability $\geq 1-\delta$,
    \begin{align*}
        \epsilon_{\gT}(\hat{h}) &\leq \epsilon_{\hat{\gS}}(\hat{h}) + \hat{\Delta}(\hat{h}, \Bar{h}) + \sqrt{\frac{(n_S + 4n_T) \log \frac{1}{\delta}}{2 n_S n_T}} + \lambda.
    \end{align*}
\end{theorem}
\begin{proof}
    We follow~\cite{rosenfeld2023almost} with our notion of \textit{adaptability} through the ideal joint hypothesis $y^*$. \\
    From~\autoref{lemma:error-expansion_app} and~\autoref{ass:dis2concave}, we have, $\epsilon_{\gT}(h) \leq \epsilon_{\gS}(h) +\Delta(h, \Bar{h}) + \lambda$.

    To upper bound the first two terms using empirical estimates, we define the following random variables:
    \begin{align*}
        r_{\gS, i} = 
        \begin{cases}
        0,& \Bar{h}(x_i) = y^\gS_i,\\
        \frac{1}{n_S},& \Bar{h}(x_i) = \hat{h}(x_i) \neq y^\gS_i,\\
        \frac{-1}{n_S},& \Bar{h}(x_i) \neq \hat{h}(x_i) = y^\gS_i , 
        \end{cases}
        \qquad
        r_{\gT, i} = \frac{\mathbf{1}\{\hat{h}(x_i) \neq \Bar{h}(x_i)\}}{n_T}
    \end{align*}
    Here, $y^\gS_i$ is the source ground truth label for the $i^\text{th}$ sample (according to $y^*_\gS$).
    Consider the following sums:
    \begin{align*}
        \sum_{\gS} r_{\gS,i} = \frac{1}{n_\gS}\sum_{\gS}[\mathbf{1}\{\hat{h}(x_i) \neq y^\gS_i\} - \mathbf{1}\{\hat{h}(x_i) \neq \Bar{h}(x_i)\}] &= \epsilon_{\hat{\gS}}(\hat{h}, y^*_\gS) - \epsilon_{\hat{\gS}}(\hat{h}, \Bar{h}),\\
        \sum_{\gT} r_{\gT, i} =\frac{1}{n_\gT}\sum_{\gT}[\mathbf{1}\{\hat{h}(x_i) \neq \Bar{h}(x_i)\}] &= \epsilon_{\hat{\gT}}(\hat{h}, \Bar{h}).
    \end{align*}
    Their sum $\epsilon_{\hat{\gS}}(\hat{h}, y^*_\gS) - \epsilon_{\hat{\gS}}(\hat{h}, \Bar{h}) + \epsilon_{\hat{\gT}}(\hat{h}, \Bar{h}) = \epsilon_{\hat{\gS}}(\hat{h}) + \hat{\Delta}(\hat{h}, \Bar{h})$ gives us the empirical estimate of the first two terms. The sum of their corresponding population terms: 
    \begin{align*}
        \E_\gS\left[ \sum_{\gS} r_{\gS,i} \right] &= \epsilon_{\gS}(\hat{h}) - \epsilon_{\gS}(\hat{h}, \Bar{h}),\\
        \E_\gT\left[ \sum_{\gT} r_{\gT, i} \right] &= \epsilon_{\gT}(\hat{h}, \Bar{h}),
    \end{align*}
    and $\lambda$ gives us the bound: $\epsilon_{\gT}(h) = \epsilon_{\gS}(\hat{h}) + \Delta(\hat{h}, \Bar{h}) + \lambda$.
    Applying Hoeffding's inequality: the probability that the expectation exceeds their sum by $t$ is no more than $\exp\left(-\frac{2t^2}{n_S \left(\frac{2}{n_S}\right)^2 + n_T \left(\frac{1}{n_T}\right)^2}\right)$ and solving for $t$ completes the proof.
\end{proof}

Note that the $\lambda$ term is incalculable in our setting (without access to target labels), so this bound only provides a qualitative relationship to the target risk. In practice, if $\lambda$ is estimated through held out target samples, the adjustment due to finite sampling (through Hoeffding's inequality) will also change accordingly.

\begin{theorem}
    \label{thm:splitbound_app}
    Under~\autoref{ass:dis2concave},
    with probability $\geq 1-\delta$,
    \begin{align*}
        \epsilon_{\gT}(\hat{h}) &\leq \epsilon_{\hat{\gS}}(\hat{h}) + \hat\Delta(\hat{h}, \Bar{h}, \alpha) + \underline{\hat\Delta}(h, \Bar{h}, \alpha) + \sqrt{\frac{(n_S + 4n_T) \log \frac{1}{\delta}}{2 n_S n_T}} + \lambda.
    \end{align*}
\end{theorem}

\begin{proof}
    As $\Delta(h, h') = \Delta(h, h', \alpha) + \underline{\Delta}(h, h', \alpha)$ for any $h, h'$ and $\alpha$, the random variables defined in the proof for~\autoref{thm:dis2bound_app} can be split in two mutually exclusive and exhaustive sets (overlapping and non-overlapping regions). The rest of the proof follows similarly.
\end{proof}



\begin{theorem}[\odd Bound]
    \label{thm:oddbound_app}
    Under~\autoref{ass:oddconcave} and~\autoref{ass:overlap},
    with probability $\geq 1-\delta$,
    \begin{align*}
        \epsilon_{\gT}(\hat{h}) &\leq \epsilon_{\hat{\gS}}(\hat{h}) + \hat\Delta(\hat{h}, \Bar{h}, \alpha) +
        \sqrt{\frac{(n_S + 4n_T) \log \frac{1}{\delta}}{2 n_S n_T}}
    \end{align*}
\end{theorem}

\begin{proof}
    With the common $y^*$, $\lambda$ goes to zero. Hence,~\autoref{lemma:error-expansion_app} implies:
    $\epsilon_{\gT}(h) \leq \epsilon_{\gS}(h) +\Delta(h, \Bar{h})$.
    
    
    

    We now define the following random variables:
    \begin{align*}
        r_{\gS \setminus \gD_\alpha,i} = 
        \begin{cases}
        0,& \Bar{h}(x_i) = y^*_i,\\
        \frac{1}{n_S},& \Bar{h}(x_i) = \hat{h}(x_i) \neq y^*_i,\\
        \frac{-1}{n_S},& \Bar{h}(x_i) \neq \hat{h}(x_i) = y^*_i , 
        \end{cases}
        \qquad
        r^\gS_{\gD_\alpha,i} = 
        \begin{cases}
        0,& \Bar{h}(x_i) = y^*_i,\\
        \frac{1}{n_S},& \Bar{h}(x_i) = \hat{h}(x_i) \neq y^*_i,\\
        \frac{-1}{n_S},& \Bar{h}(x_i) \neq \hat{h}(x_i) = y^*_i, 
        \end{cases}
    \end{align*}
    \begin{align*}
        r_{\gT \setminus \gD_\alpha, i} = \frac{\mathbf{1}\{\hat{h}(x_i) \neq \Bar{h}(x_i)\}}{n_T}
        \qquad
        r^\gT_{\gD_\alpha, i} = \frac{\mathbf{1}\{\hat{h}(x_i) \neq \Bar{h}(x_i)\}}{n_T}
    \end{align*}
    Now we have the following empirical sums:
    \begin{align*}
        \sum_{\gS \setminus \gD_\alpha} r_{\gS \setminus \gD_\alpha,i} = \frac{1}{n_\gS}\sum_{\gS \setminus \gD_\alpha}[\mathbf{1}\{\hat{h}(x_i) \neq y^*_i\} - \mathbf{1}\{\hat{h}(x_i) \neq \Bar{h}(x_i)\}] &= \epsilon^\gS_{\widehat{\gS \setminus \gD_\alpha}}(\hat{h}, y^*) - \epsilon^\gS_{\widehat{\gS \setminus \gD_\alpha}}(\hat{h}, \Bar{h}),\\
        \sum_{\gT \setminus \gD_\alpha} r_{\gT \setminus \gD_\alpha, i} =\frac{1}{n_\gT}\sum_{\gT \setminus \gD_\alpha}[\mathbf{1}\{\hat{h}(x_i) \neq \Bar{h}(x_i)\}] &= \epsilon^\gT_{\widehat{\gT \setminus \gD_\alpha}}(\hat{h}, \Bar{h}),\\
        \sum_{\gD_\alpha} r^\gS_{\gD_\alpha,i} = \frac{1}{n_\gS}\sum_{\gD_\alpha}[\mathbf{1}\{\hat{h}(x_i) \neq y_i\} - \mathbf{1}\{\hat{h}(x_i) \neq \Bar{h}(x_i)\}] &= \epsilon^\gS_{\hat{\gD_\alpha}}(\hat{h}, y^*) - \epsilon^\gS_{\hat{\gD_\alpha}}(\hat{h}, \Bar{h}),\\
        \sum_{\gD_\alpha} r^\gT_{\gD_\alpha, i} =\frac{1}{n_\gT}\sum_{\gD_\alpha}[\mathbf{1}\{\hat{h}(x_i) \neq \Bar{h}(x_i)\}] &= \epsilon^\gT_{\hat{\gD_\alpha}}(\hat{h}, \Bar{h}),
    \end{align*}
    and their corresponding population terms: 
    \begin{align*}
        \E_\gS\left[ \sum_{\gS \setminus \gD_\alpha} r_{\gS \setminus \gD_\alpha,i} \right] &= \epsilon^\gS_{\gS \setminus \gD_\alpha}(\hat{h}, y^*) - \epsilon^\gS_{\gS \setminus \gD_\alpha}(\hat{h}, \Bar{h}),\\
        \E_\gT\left[ \sum_{\gT \setminus \gD_\alpha} r_{\gT \setminus \gD_\alpha, i} \right] &= \epsilon^\gT_{\gT \setminus \gD_\alpha}(\hat{h}, \Bar{h}),\\
        \E_\gS\left[ \sum_{\gD_\alpha} r^\gS_{\gD_\alpha,i} \right] &= \epsilon^\gS_{\gD_\alpha}(\hat{h}, y^*) - \epsilon^\gS_{\gD_\alpha}(\hat{h}, \Bar{h}),\\
        \E_\gT\left[ \sum_{\gD_\alpha} r^\gT_{\gD_\alpha, i} \right] &= \epsilon^\gT_{\gD_\alpha}(\hat{h}, \Bar{h}).
    \end{align*}
    The sum of these terms
    \begin{align*}
        &= \epsilon^\gS_{\gS \setminus \gD_\alpha}(\hat{h}, y^*) - \epsilon^\gS_{\gS \setminus \gD_\alpha}(\hat{h}, \Bar{h}) + \epsilon^\gT_{\gT \setminus \gD_\alpha}(\hat{h}, \Bar{h}) + \epsilon^\gS_{\gD_\alpha}(\hat{h}, y^*) - \epsilon^\gS_{\gD_\alpha}(\hat{h}, \Bar{h}) + \epsilon^\gT_{\gD_\alpha}(\hat{h}, \Bar{h})\\
        &= (\epsilon^\gS_{\gS \setminus \gD_\alpha}(\hat{h}, y^*) + \epsilon^\gS_{\gD_\alpha}(\hat{h}, y^*)) + (\epsilon^\gT_{\gT \setminus \gD_\alpha}(\hat{h}, \Bar{h}) - \epsilon^\gS_{\gS \setminus \gD_\alpha}(\hat{h}, \Bar{h})) + (\epsilon^\gT_{\gD_\alpha}(\hat{h}, \Bar{h}) - \epsilon^\gS_{\gD_\alpha}(\hat{h}, \Bar{h}))\\
        &= \epsilon_{\gS}(\hat{h}, y^*) + \Delta(\hat{h}, \Bar{h}, \alpha) + (\epsilon^\gT_{\gD_\alpha}(\hat{h}, \Bar{h}) - \epsilon^\gS_{\gD_\alpha}(\hat{h}, \Bar{h}))\\
        &\leq \epsilon_{\gS}(\hat{h}, y^*_\gS) + \Delta(\hat{h}, \Bar{h}, \alpha) + \underline{\Delta}(\hat{h}, \Bar{h}, \alpha)\\
        &\leq \epsilon_{\gS}(\hat{h}, y^*_\gS) + \Delta(\hat{h}, \Bar{h}, \alpha) \qquad \left(\underline{\Delta}(\hat{h}, \Bar{h}, \alpha) \leq 0 \text{ from~\autoref{ass:overlap}.}\right)
    \end{align*}
    Now, applying Hoeffding's inequality completes the proof.

\end{proof}

\paragraph{Remark:}~\autoref{ass:overlap} bounds the target risk only in terms of non-overlapping disagreement discrepancy. It may be counter-intuitive to think that target disagreement is bounded by source disagreement. However, we have shown that $\Delta(\hat{h}, \bar{h}, \alpha)$ accounts for nearly all of $\Delta(\hat{h}, \bar{h})$ (\autoref{fig:synthetic_main}) for $\bar{h}$ found through either \disdis or \odd. This implies that even if the target disagreement exceeds source disagreement, it doesn't do it by a lot. Therefore, even without~\autoref{ass:overlap}, the $\underline{\Delta}$ term from~\autoref{thm:splitbound} should be expected to be negligible and inconsequential in most practical cases.

\section{Experiment Details}
\label{app:expdetails}

\subsection{Dataset details}
\paragraph{Synthetic Datasets:}
The Gaussians are randomly initialized with varying means and covariance matrices. The target gaussian is brought close to the source gaussian by translating its mean using $\mu_\gT \gets \mu_\gT + (\mu_\gS - \mu_\gT) * \text{overlap factor}$. For a point $X = (x_1, x_2)$, its class label ($y^*$) is decided using a complex function:  
\[
y^*(X) =
\begin{cases} 
  0, & \text{if } x_1 \leq \cos(a\sin(bx_2) + ce^{dx_2} + \frac{x_2^2 + 2x_2 - 5}{2}) + \epsilon \\
  1, & \text{otherwise}
\end{cases}
\]
where, $a, b, c, d$ are chosen randomly from a small range and $\epsilon$ is small random noise to make the decision boundary noisy. This allowed for simple as well as extremely complex decision surfaces, making for a suitable test bed for our method and its comparison with \disdis. We show 5 more samples of these datasets in~\autoref{fig:synth_samples}.

\begin{figure}[h]
    \centering
    \includegraphics[width=0.19\linewidth]{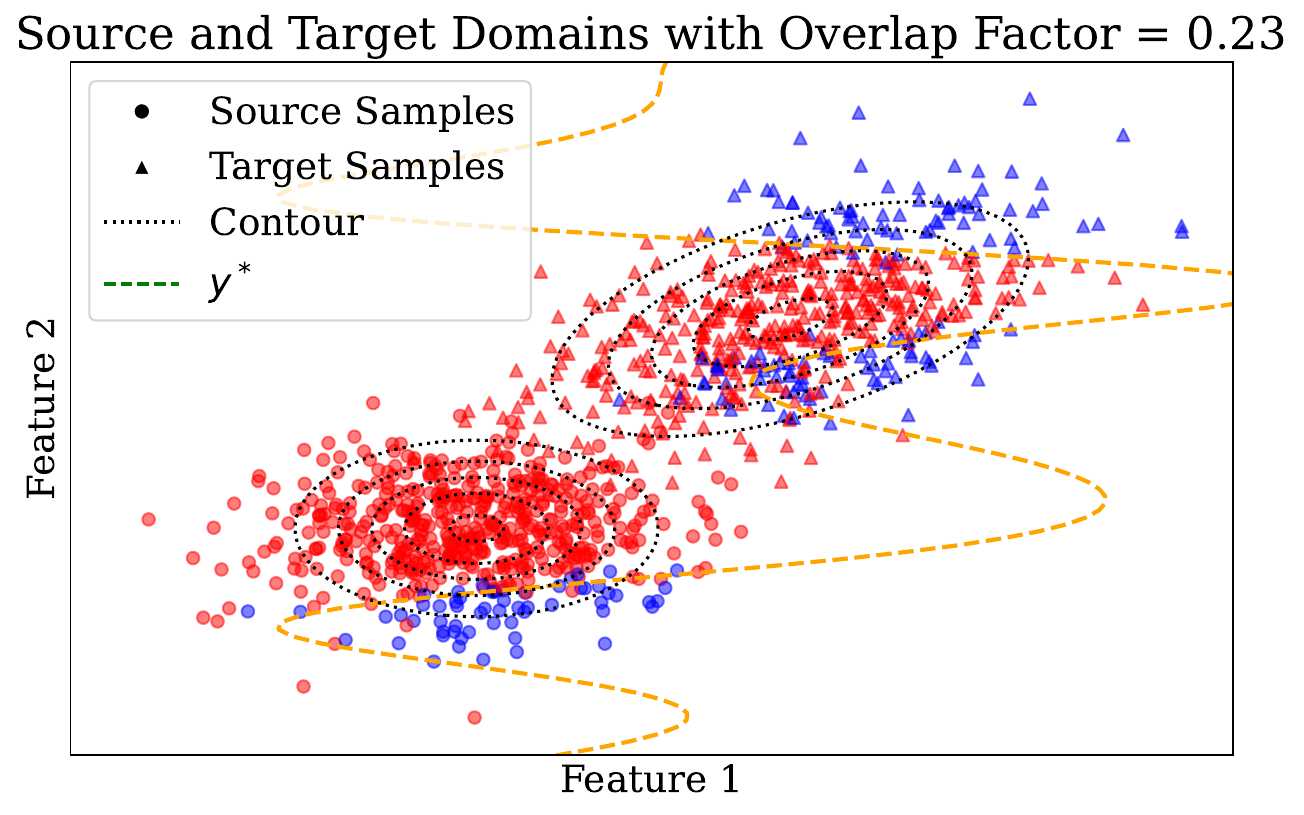}
    \includegraphics[width=0.19\linewidth]{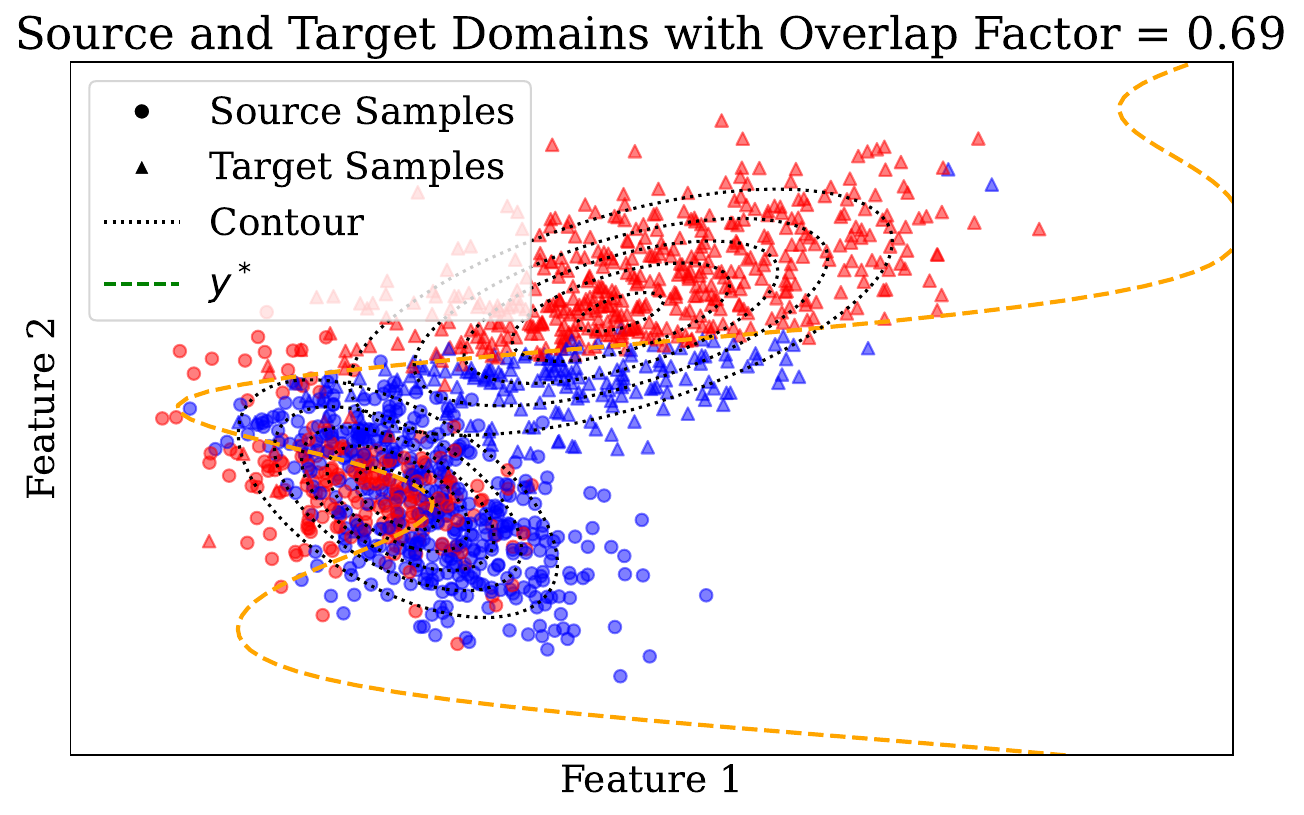}
    \includegraphics[width=0.19\linewidth]{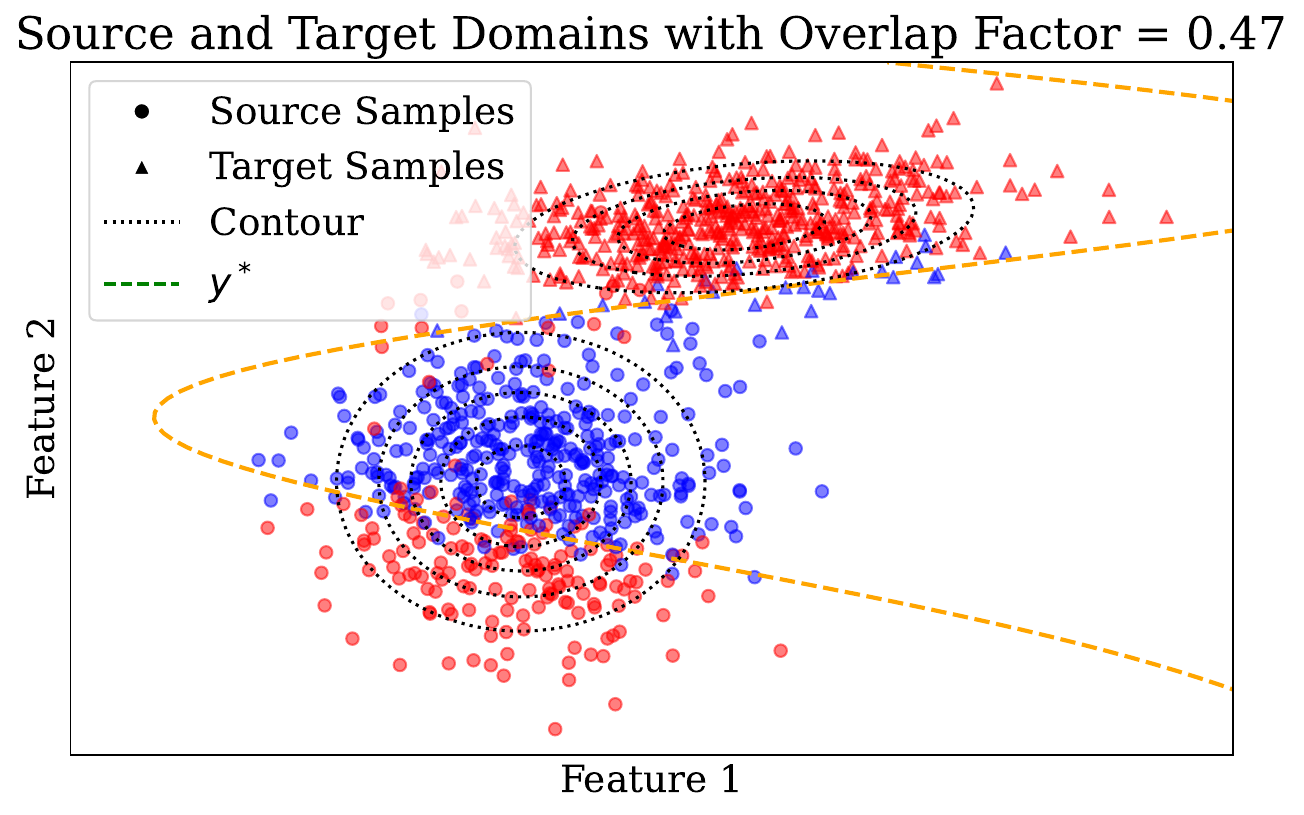}
    \includegraphics[width=0.19\linewidth]{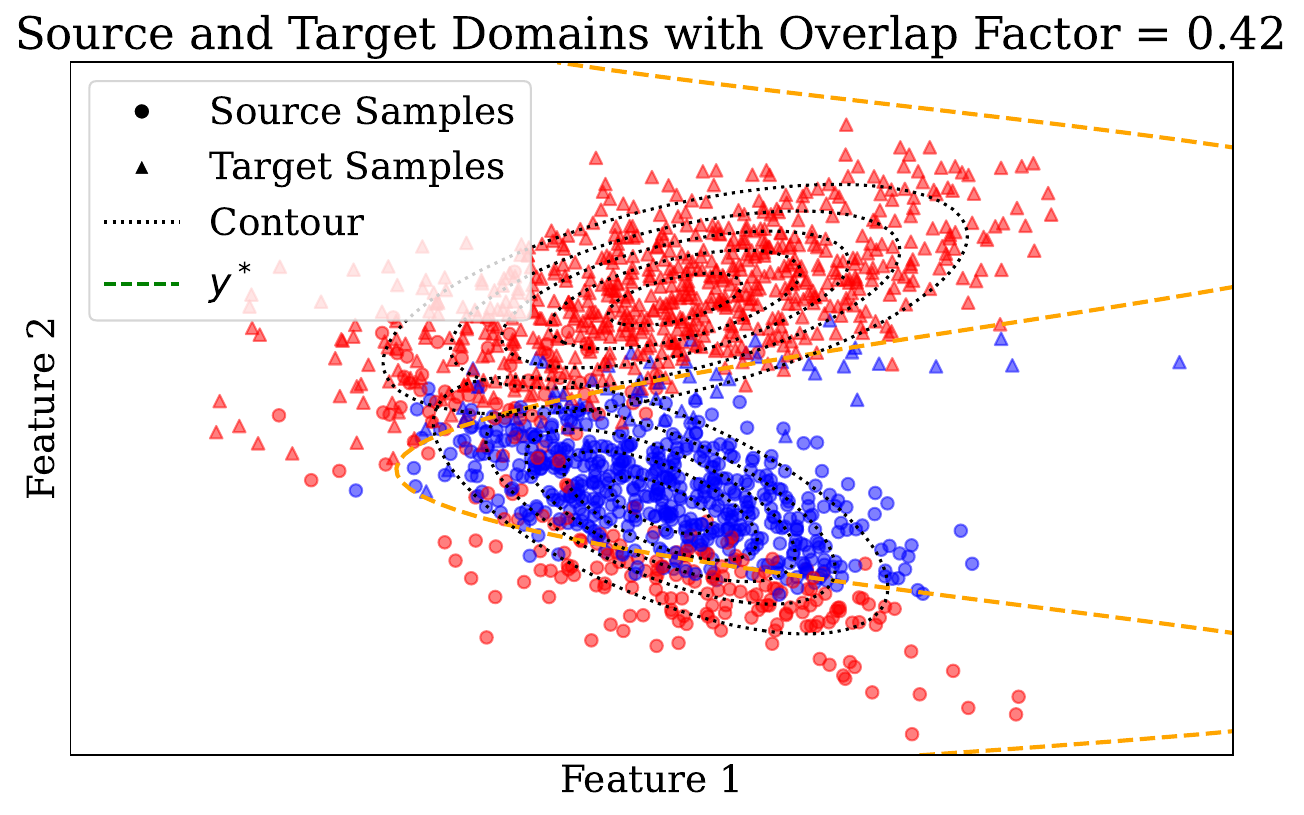}
    \includegraphics[width=0.19\linewidth]{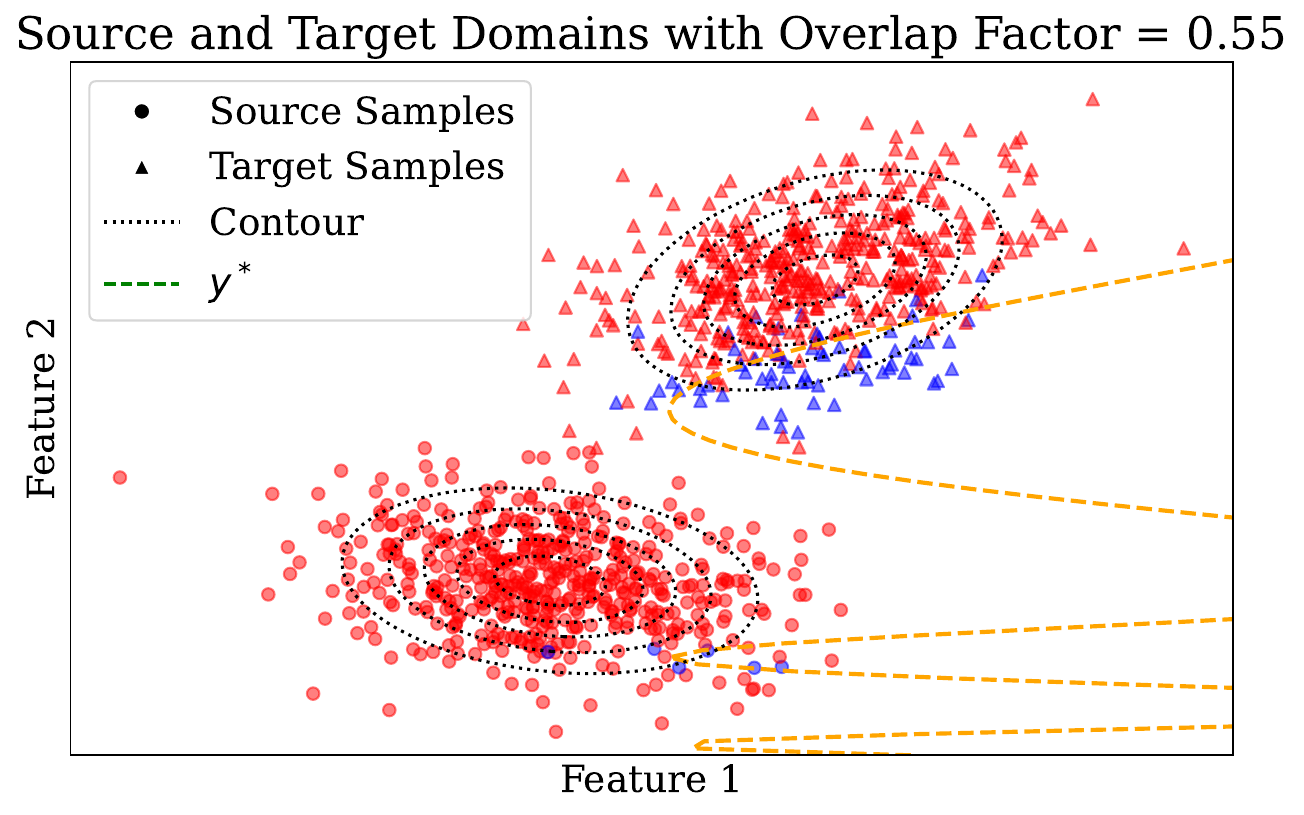}
    \caption{Variations in synthetic datasets.}
    \label{fig:synth_samples}
\end{figure}

\paragraph{Real Datasets:}
We use the datasets provided by~\citet{rosenfeld2023almost} on their \href{https://github.com/erosenfeld/disagree_discrep/tree/master}{github repository}. This includes CIFAR10  [4 shift variations], CIFAR100 [4 shift variations], DomainNet [3 shift variations], Entity13 [3 shift variations], Entity30 [3 shift variations], FMoW [2 shift variations], Living17 [3 shift variations], NonLiving26 [3 shift variations], OfficeHome [3 shift variations] and Visda [2 shift variations]. All datasets were trained with different deep neural network backbones as described in the \disdis paper (Appendix A) with ERM, DANN, CDANN, BN-adapt and FixMatch methods. The deep representations of the trained networks are provided directly for download making replication seamless. 

\paragraph{CivilComments:} We used CivilComments~\cite{borkan2019nuanced} to test if our method improves over \disdis in natural language domain, which it had not been tested on before. This dataset contains sentences labeled \textit{toxic} or not. We created source and target domains using subpopulation shift in this dataset. We filtered all rows with the \textit{black} label True into one set (target) and the rest in another (source), and evenly balanced the dataset to avoid class imbalance issues in evaluation. Train Set: $\sim$ 27k samples in each class, Validation Set: $\sim$ 4.5k samples in each class.  We collected the BERT~\citep{devlin2018bert} embeddings of all sentences in the source and target domain and trained a linear model on the source domain, which achieved a source validation accuracy of $\sim75\%$. The source model achieved a target validation accuracy of $\sim65\%$, which would be the target to predict by our bound. We found that the \disdis-based bound predicted an accuracy of $\sim58\%$, while \odd improved it to $\sim63\%$.

\subsection{More experimental details}
\paragraph{Training domain classifiers:}  
For each experiment: (dataset, shift, training method) combination, we train a small 3-layer MLP, with number of neurons $=$ the dimensionality of the representation. We train the network on a balanced training set consisting of equal number of samples from both source and target domains. We randomly subsample the majority domain to make both classes balanced to not skew the classifier towards one of the classes. We train the classifier with Adam Optimizer, with learning rate $=10^{-4}$, until convergence (difference in loss $< 10^{-4}$ for $10$ epochs).

\subsection{More experimental results}

\paragraph{Results on non-DA algorithms:} It is noticeable that the MAE for domain-adversarial algorithm based representations is much better than the other. We plot the point-wise bound estimates (using logits) for all DA-method predictions in the appendix~\autoref{fig:real_logits_DA} and find that like \disdis, \odd overestimates accuracy in many cases. This is expected as explained in~\autoref{para:da}, but our method still maintains a significantly higher coverage compared to baselines.


\begin{figure}[b]
    \centering
    \includegraphics[width=0.7\linewidth]{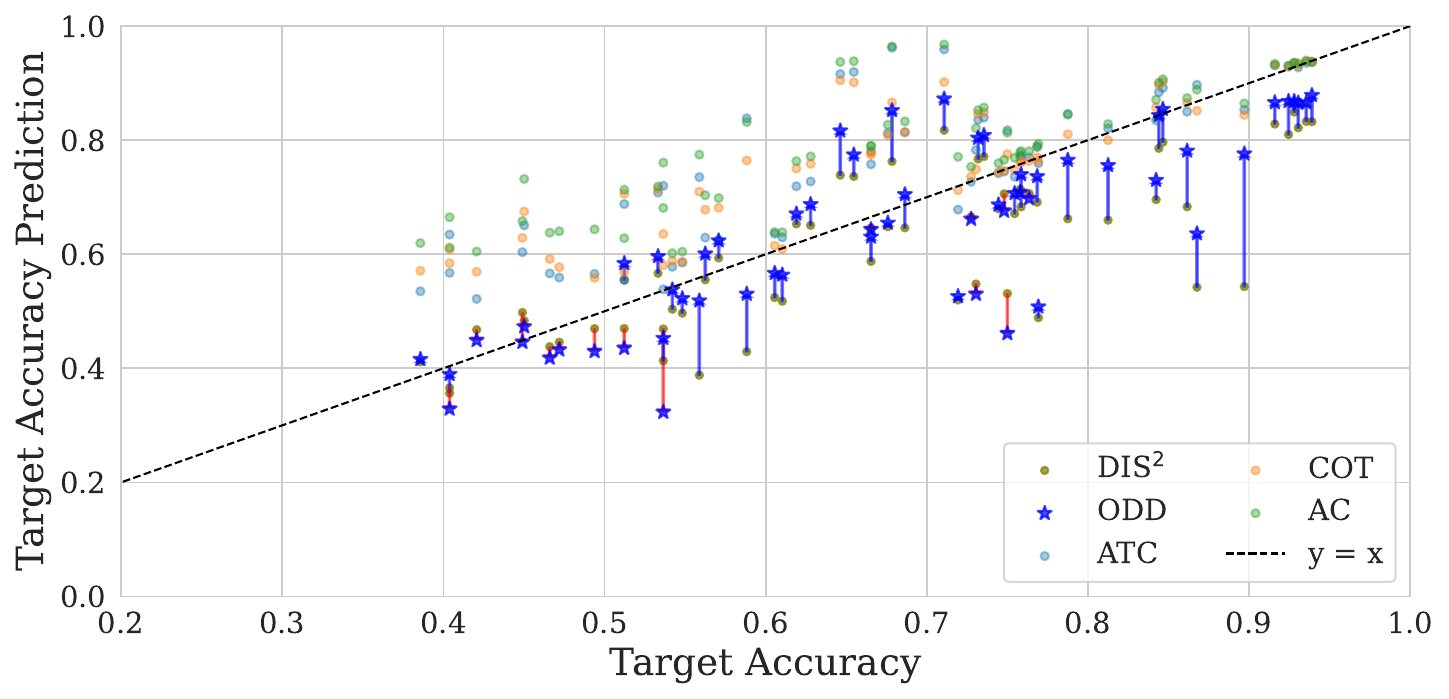}
    \caption{Like \disdis, \textbf{DA?}\ding{55} based predictions overestimate accuracy in many cases but still by a smaller margin compared to other baselines.}
    \label{fig:real_logits_DA}
\end{figure}


\paragraph{Discounting loss of overlapping source samples:}We ran an additional experiment on the real datasets where we weighted the source samples in addition to the target samples while optimizing for the critic. We found no significant difference in MAE under this setting (MAE without source weighting: 0.1224, MAE with source weighting: 0.1244). This is probably because we are searching in the simple space of linear models, and ignoring a few samples does not impact the selection of the critic much on average. 

\paragraph{Measuring quality of overlap estimation:} To measure the quality of overlap estimation, we split the source and target data (of a few studied datasets) into train and dev sets, and early stopped the domain classifier training if the dev set loss did not fall for 10 iterations in a row. As we do not have access to true densities, the dev set performance acts as a proxy for the quality of overlap estimation.  We tried multiple sized MLPs and found that the mean dev set accuracy across datasets was almost the same (apart from the linear classifier, which probably underfits in some cases). For details, see~\autoref{tab:dcab} where we also measure the expected calibration error (ECE)~\citep{guo2017calibration} of the model. Although we could improve ODD-based bounds by tuning the classifier’s hyperparameters for each individual dataset, we used a single setting (3 layer deep, num\_features wide network) for all experiments to reflect the robustness of our approach. 

\begin{table}[h]
\centering
\caption{We find that training domain classifiers on representations is fairly robust to the model architecture.}
\resizebox{0.8\textwidth}{!}{%
\begin{tabular}{ccccccc}
\textbf{num\_layers} & \textbf{width}     & \textbf{Accuracy (mean)} & \textbf{Accuracy (std)} & \textbf{ECE (mean)} & \textbf{ECE (std)} & \textbf{Final MAE} \\
1                    & -                  & 0.66                     & 0.15                    & 0.09                & 0.08               & 0.1234             \\
2                    & num\_features      & 0.71                     & 0.16                    & 0.11                & 0.09               & 0.1269             \\
2                    & num\_features // 2 & 0.69                     & 0.17                    & 0.10                & 0.09               & 0.1236             \\
3                    & num\_features      & 0.71                     & 0.15                    & 0.15                & 0.10               & 0.1223             \\
3                    & num\_features // 2 & 0.70                     & 0.16                    & 0.14                & 0.09               & 0.1282             \\
4                    & num\_features      & 0.70                     & 0.16                    & 0.19                & 0.12               & 0.1229             \\
4                    & num\_features // 2 & 0.69                     & 0.16                    & 0.18                & 0.11               & 0.1255            
\end{tabular}%
}
\label{tab:dcab}
\end{table}




\end{document}